\documentclass{article}

\usepackage{microtype}
\usepackage{graphicx}
\usepackage{subcaption}
\usepackage{booktabs} 

\usepackage{amsmath,amsfonts,bm}

\def\eqref#1{equation~\ref{#1}}

\def\1{\bm{1}}

\DeclareMathAlphabet{\mathsfit}{\encodingdefault}{\sfdefault}{m}{sl}
\SetMathAlphabet{\mathsfit}{bold}{\encodingdefault}{\sfdefault}{bx}{n}

\usepackage{bm}
\usepackage{hyperref}
\usepackage{url}
\usepackage{booktabs}
\usepackage{multirow}
\usepackage{amsthm}
\usepackage{graphicx} 
\usepackage{colortbl} 
\usepackage{arydshln}

\usepackage{hyperref}

\usepackage[accepted]{icml2026}

\usepackage{amsmath}
\usepackage{amssymb}
\usepackage{mathtools}
\usepackage{amsthm}

\usepackage[capitalize,noabbrev]{cleveref}

\theoremstyle{plain}
\newtheorem{theorem}{Theorem}[section]
\newtheorem{proposition}[theorem]{Proposition}

\newtheorem{corollary}[theorem]{Corollary}
\theoremstyle{definition}
\newtheorem{definition}[theorem]{Definition}

\theoremstyle{remark}
\newtheorem{remark}[theorem]{Remark}

\usepackage[textsize=tiny]{todonotes}

\icmltitlerunning{Hierarchical Decision Making with Structured Policies: A Principled Design via Inverse Optimization}

\begin{document}

\twocolumn[
\icmltitle{Hierarchical Decision Making with Structured Policies: A Principled Design via Inverse Optimization}


\begin{icmlauthorlist}
\icmlauthor{Yuexuan Wang}{aff1}
\icmlauthor{Jingyuan Zhou}{aff2}
\icmlauthor{Kaidi Yang}{aff2}
\end{icmlauthorlist}

\icmlaffiliation{aff1}{Institute of Operations Research and Analytics, National University of Singapore}
\icmlaffiliation{aff2}{Department of Civil and Environmental Engineering, National University of Singapore}

\icmlcorrespondingauthor{Kaidi Yang}{kaidi.yang@nus.edu.sg}

\icmlkeywords{Machine Learning, ICML}

\vskip 0.3in
]

\printAffiliationsAndNotice{}
\begin{abstract}


Hierarchical decision-making frameworks are pivotal for addressing complex control tasks,  enabling agents to decompose intricate problems into manageable subgoals. Despite their promise, existing hierarchical policies face critical limitations: (i) reinforcement learning (RL)-based methods struggle to guarantee strict constraint satisfaction, and (ii) optimal control (OC)-based approaches often rely on myopic and computationally prohibitive formulations. To reconcile these trade-offs, hierarchical RL-OC architectures have emerged as a promising paradigm. However, the formulation of the lower-level optimization within these frameworks remains underexplored, often relying on heuristic or myopic objectives. In this work, we propose a principled framework that systematically integrates upper-level goal abstraction with structured lower-level decision making. We adopt an inverse optimization approach to inform the structure of the lower-level problem from expert demonstrations, ensuring that the objective of the lower-level policy remains aligned with the overall long-term task goal. To validate the approach, our framework is evaluated on distinct decision making tasks: network-based resource allocation and continuous collision avoidance. Empirical results demonstrate that our method consistently outperforms strong baselines based on end-to-end RL, learning-augmented optimal control, and existing hierarchical RL approaches in both efficiency and decision quality.

\end{abstract}

\section{Introduction}

Real-time decision-making in cyber-physical systems, such as robotics, power grids, and transportation~\citep{jendoubi2023multi,zhou2024enhancing,liang2025interaction}, is inherently challenging due to high-dimensional state and action spaces and complex physical constraints. Existing solutions largely stem from optimal control (OC) and deep reinforcement learning (RL). OC-based methods aim to optimize system performance over infinite or long horizons while ensuring stability and feasibility. 
These methods are well-suited for safety-critical systems due to their theoretical guarantees, but may scale poorly in high-dimensional or nonlinear settings~\citep{yu2013model}. In contrast, RL-based approaches directly learns a policy from interactions with the environment, which scale well to complex tasks. Nevertheless, these approaches require extensive training and lack safety or constraint satisfaction guarantees due to their black-box nature~\citep{zhao2023state, wang2020deep}. These trade-offs have motivated growing interest in combining OC-based and RL-based methods to exploit the strengths of both paradigms. 

A promising approach for combining OC- and RL-based methods is through a hierarchical architecture that decomposes decision-making into two sequential subproblems~\citep{lew2023robotic,karnchanachari2020practical}. The respective strengths of RL and OC fit naturally into this framework: the upper-level employs an RL policy for strategic planning, such as generating subgoals, while the lower level uses OC to ensure safe and feasible execution. This hierarchical architecture not only enhances scalability and feasibility but also aligns with human cognition, as humans tend to perform abstract planning guided by intrinsic motivation, grounded by fast, lower-level execution ~\citep{aubret2019survey}.

Despite the promise of hierarchical RL–OC frameworks, the formulation of the lower-level optimization problem remains underexplored. The lower-level controller must be both computationally efficient and aligned with the upper-level goals, since a poorly designed formulation may inadvertently exclude high-quality solutions. Existing approaches have several limitations. First, most hierarchical methods adopt long-horizon OC formulations at the lower level to preserve stability and feasibility guarantees ~\citep{songlearning, cheng2024hierarchical,landgraf2022hierarchical}. However, such long-horizon OC formulations are well known to incur prohibitive computational complexity for real-time applications \citep{karamanakos2014direct, krishnamoorthy2020adaptive}. Second, recent efforts leverage single-step OC to enhance computation efficiency by directly generating the value or constraints on the desired next state~\citep{gammelli2023graph,schmidt2024offline}. Nevertheless, these formulations typically rely on myopic objectives without an appropriately designed formulation, leading to suboptimal trajectories~\citep{rawlings2020model, lowrey2018plan}. The above challenges highlight the importance of formulating the lower-level optimization problem in a way that reduces sub-optimality while simultaneously ensuring tractable computational complexity.

To address these limitations, we propose an inverse optimization-guided approach that systematically informs the design of the lower-level policy using a small set of expert demonstrations (Figure \ref{fig:method}), which provide valuable insights. Rather than relying on manually specified objectives, we cast the construction of the lower-level cost function as an inverse optimization problem, aiming to recover a formulation under which expert decisions are (approximately) optimal. For a special class of lower-optimization problems with linear cost functions, we provide a theoretical characterization of the conditions under which the expert demonstrations are optimal. 


\begin{figure*}
    \centering
\includegraphics[width=0.7\linewidth]{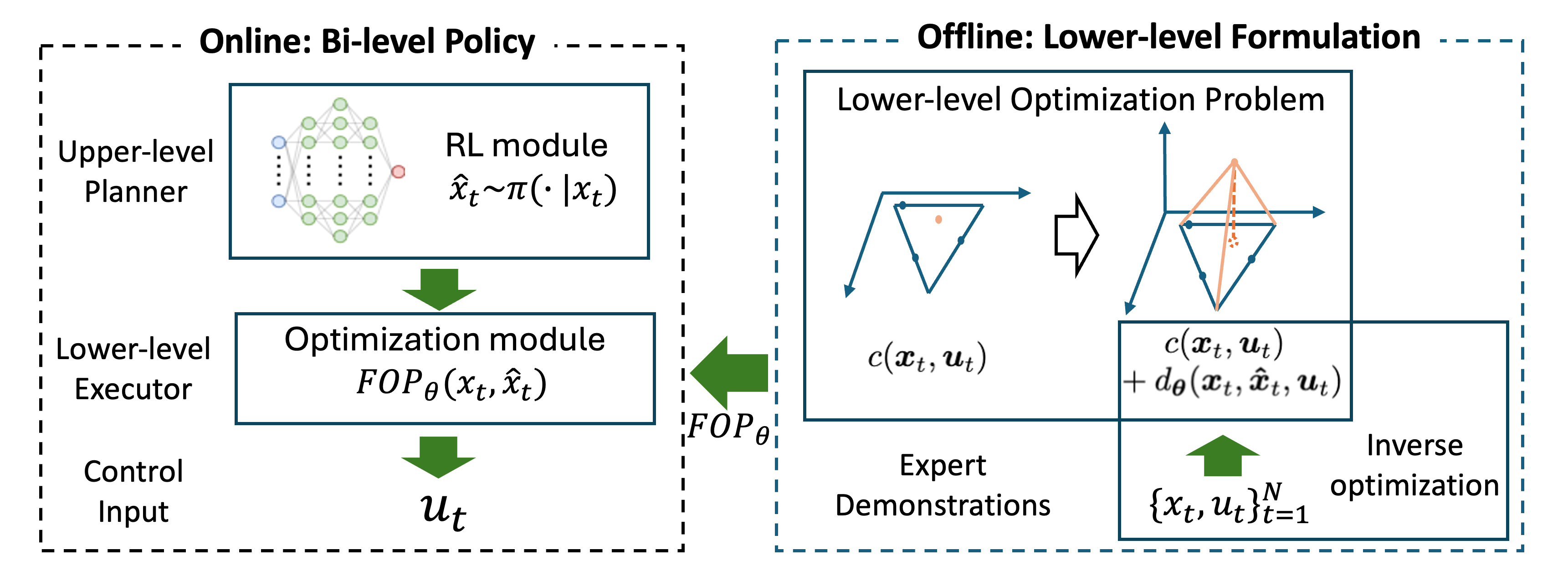}
    \caption{We propose an RL-OC hierarchical decision-making framework with lower-level policy informed by inverse optimization.}
    \label{fig:method}
\end{figure*}

Once the formulation is established, we develop computationally efficient methods to solve the resulting inverse problem and integrate the learned lower-level structure into a hierarchical RL–OC framework. We evaluate the proposed approach on three representative decision-making tasks from various fields: autonomous vehicle rebalancing, supply chain inventory management, and mobile robot navigation. The improvements in learning the formulation are validated from quantitative and qualitative perspectives. In light of the above discussion, we summarize the main contributions of this work as follows:\vspace{-0.6em}
\begin{itemize}
    \item We propose an inverse optimization-based approach to systematically construct the lower-level optimization formulation within a hierarchical RL-OC framework.
    \item We provide theoretical analysis for a special class of problems with broad applications, proposing a tractable cost structure and efficient inverse optimization formulation which ensures inverse-feasibility, forward-stability, and computational tractability.
    \item We demonstrate the effectiveness of the proposed framework on several scenarios from various fields, showcasing its practical relevance and potential impact in real-world applications.
\end{itemize}
\section{Related Work}
This work is related to the literature on hierarchical structured control policies. Depending on whether learning-based or model-based approaches are used at each level, existing works can be broadly classified into two categories: (i) hierarchical reinforcement learning (HRL) that employs RL at each level, and (ii) learning-based optimal control, e.g., frameworks integrating RL and Model Predictive Control (MPC), whereby a upper-level policy learns desired states or goals, and a simplified lower-level MPC ensures safe and feasible execution.

\noindent \textbf{Hierarchical Reinforcement Learning.}
HRL decomposes a complex, difficult-to-solve problem into multiple simpler, smaller problems by setting subgoals \citep{kulkarni2016hierarchical, vezhnevets2017feudal,ma2021hierarchical,xie2021hierarchical,eppe2022intelligent,qi2022hierarchical,huang2022lstm,gu2023safe,mao2024integrating,luo2024goal,zhang2024spatial,hirt2024stability}.  
We focus on how hierarchical policies utilize various forms of intrinsic motivation by setting subgoals. For example, \citet{naveed2021trajectory} uses a high-level policy to choose maneuvers for autonomous driving, while a low-level planner generates waypoints accordingly. \citet{vezhnevets2017feudal} proposes Feudal Networks, where a manager sets abstract goals that are conveyed to and enacted by the Worker module. Another common choice is to represent goals as desired states; e.g., \citet{nachum2018data} treats goal states as high-level actions and rewards the low-level policy for reaching them. In recent years, a growing body of work has investigated how to define subgoals and how to search efficiently within the subgoal space \citep{liu2021hierarchical, ma2023human}. Nevertheless, existing studies do not explicitly model how subgoals induce final actions via a well-defined lower-level optimization problem, often relying instead on implicit or black-box policy mappings. Moreover, as previously discussed, incorporating constraints to ensure safety strictly is challenging in reinforcement learning framework. 
In next section, we focus on how learning, especially reinforcement learning, interplays with optimization in the previous literature.


\noindent \textbf{Learning-based OC.} In the control community, various real-world control problems are solved by using learning-based OC. In many existing works, learning-based methods are often applied to learn cost functions or system dynamics \citep{lenz2015deepmpc,coulson2019data,hewing2020learning,dogan2023regret,zhang2024inverse,lu2024mpc,zhang2024inverse,dinkla2026closed}. 
However, solving optimal control problems in real time still poses challenges when long control horizons are used, due to the high dimensionality of variables and the complexity of constraints. 
Prior work has attempted to reduce the computational burden by approximating the long-horizon OC problem with a single-step formulation and learning terminal cost to alleviate myopic behavior \citep{abdufattokhov2021learning,alsmeier2024neural,zhang2024inverse}. 

However, learning a terminal cost alone can be insufficient when the horizon is aggressively shortened, as the terminal term would need to encode most of the long-term planning signal, which is often highly state-dependent, and may be difficult to represent. This motivates RL-OC frameworks. \citet{gammelli2023graph} propose to leverage reinforcement learning to learn upper-level actions to shorten the control horizon of network flow control problem. In this framework, reinforcement learning is applied to generate desired next states that guide the lower-level executor toward maximizing cumulative rewards. However, the lower-level optimization problem still hinges on myopic objectives and does not rely on an explicitly and appropriately designed formulation. In addition, \citet{schmidt2024offline} study hierarchical RL-OC in an offline setting and focus on generating upper-level subgoals. Their framework assumes that the offline data are optimal for the lower-level control problem, which can be sometimes unrealistic and leaves the design of the lower-level OC formulation unaddressed. 

Related to this line of research, our work investigates how to formulate the lower-level optimization problem in hierarchical RL-OC framework to alleviate sub-optimality issues stemming from structural myopia. 





\section{Methodology}
\subsection{Problem Setting and Preliminary}
Let us consider a general multi-step decision-making problem formulated in (\ref{eq:problem1}).
\begin{equation}
\begin{aligned}
     \min_{\{\bm{u}_t\}_{t=0}^\infty} &\limsup_{T \to \infty} \frac{1}{T} \sum_{t=0}^{T-1} c(\bm{x}_t, \bm{u}_t)\\
     \text{s.t.} \quad & \bm{x}_{t+1} = f(\bm{x}_t) + g(\bm{x}_t) \bm{u}_t,\quad \forall t \geq 0 \\
     & \bm{x}_t \in \mathcal{X}_t,\quad \forall t \geq 0\\
     &  \bm{u}_t \in \mathcal{U}_t,\quad \forall t \geq 0
\end{aligned}
\label{eq:problem1}
\end{equation}
where $\bm{x}_t \in \mathbb{R}^n$ is the system state at time step $t$, $\bm{u}_t \in \mathbb{R}^m$ is the control input, and $c(\bm{x}_t, \bm{u}_t)$ is the stage cost function. The system dynamics are assumed to be control-affine (i.e., linear in $\bm{u}_t$) with functions $f(\cdot)$ and $g(\cdot)$. The feasible sets $\mathcal{X}_t$ and $\mathcal{U}_t$ encode admissible states and inputs, respectively. The initial state $x_0$ is known.

\begin{remark}[Generality of problem setting] \label{rmk:1}
We consider the control-affine dynamics in~(\ref{eq:problem1}) for two reasons. First, control-affine models represent a fundamental and broad class of systems in cyber-physical applications~\citep{schmitz2025excitation, kazemian2024random}. Even when a system is not canonically control-affine, prior work has shown that, for some systems, control-affine approximations can still be leveraged to enable effective control~\citep{mao2017successive, mutha1997nonlinear}. Second, this choice allows us to isolate and highlight the proposed inverse-optimization design framework, rather than on the additional algorithmic challenge of solving inverse problems for more general lower-level programs. Importantly, the framework can be extended to non-control-affine dynamics, at the cost of employing more sophisticated inverse-problem solvers, which have been extensively studied in prior work~\citep{bertsimas2015data, schaefer2009inverse}. 
\end{remark}

Due to the infinite horizon and the presence of complex constraints, solving Problem (\ref{eq:problem1}) is often computationally intractable. 
A standard workaround is a finite-horizon approximation. However, OC problems with long time horizons can still be computationally challenging for large-scale systems, which do not satisfy the real-time requirements of practical applications. To address this issue, we generalize the bi-level decision-making framework proposed by \citet{gammelli2023graph} to a more general problem setting. 

\subsection{Bi-level Framework}
The hierarchical RL-OC framework is implemented as an end-to-end system shown in Problem (\ref{eq:grouped}). 
\begin{equation} 
    \begin{aligned}
\pi^* \in \arg &\min_{\pi \in \Pi} \mathbb{E}_\tau \left[ \sum_{t=0}^\infty \gamma^t c(\bm{u}_t, \bm{x}_t)\right]\\
\text{s.t.}\quad \bm{h}_t &\sim \pi(\bm{h}_t \mid \bm{x}_t)\\
\bm{u}_t &= \mathrm{FOP}(h_t, \bm{x}_t) 
\end{aligned}
\label{eq:grouped}
\end{equation}

The overall policy $\pi^*$ composes an upper-level RL policy $\pi$ and the solution to a lower-level optimization problem $\mathrm{FOP}$. 
The upper-level policy encodes task-relevant abstract information or goals to produce the intrinsic subgoal $\bm{h}_t$. The lower-level optimization module receives both the intrinsic subgoal $h_t$ and the current state $\bm{x}_t$ to compute the control input $\bm{u}_t$ that is feasible. 

\begin{remark}[Practical requirements for RL-OC frameworks]\label{rmk:1}
    There are two requirements for Problem~(\ref{eq:grouped}). First, for real-time deployment, $\mathrm{FOP}$ should capture operational constraints and be computationally efficient to solve. Second, to enable stable training, the intrinsic subgoal $\bm{h}_t$ is preferably low-dimensional. 
\end{remark}

Following Remark~\ref{rmk:1}, a commonly used form of subgoals is a linear transformation of the desired next state $\bm{x}_t^{\text{des}}$ using a known matrix $C$, denoted by $\bm{\hat{x}}_t=C\bm{x}_t^{\text{des}}$ \citep{gammelli2023graph, schmidt2024offline}. This subgoal guides the system toward the desired state by inducing appropriate lower-level actions $u_t$. Here, in scenarios with high state dimensions, the transformation matrix $C$ serves to compress the action space of the upper-level RL by mapping the high-dimensional system state to a lower-dimensional planning space. The transformation matrix $C$ is typically set as an identity matrix in scenarios with low state dimensions. 
Overall, the bi-level decision-making framework is given in (\ref{eq:grouped2}) and (\ref{eq:problem5}). 
\begin{subequations}
    \begin{align}
\pi^* \in \arg &\min_{\pi \in \Pi} \mathbb{E}_\tau \left[ \sum_{t=0}^\infty \gamma^t c(\bm{u}_t, \bm{x}_t) \right]\\
s.t.\quad \bm{\hat{x}}_t &\sim \pi(\bm{\hat{x}}_t \mid \bm{x}_t)\\
\bm{u}_t &\in \mathrm{FOP}(\bm{x}_t,\bm{\hat{x}}_t) 
\end{align} \label{eq:grouped2}
\end{subequations}
\begin{subequations}
\begin{align}
\mathrm{FOP}(\bm{x}_t,\bm{\hat{x}}_t) &:=\  \arg\min_{\bm{u}_t} \ c^{\mathrm{FOP}}(\bm{x}_t, \bm{u}_t) \\
\text{s.t.} \quad & Cf(\bm{x}_t) + Cg(\bm{x}_t) \bm{u}_t = \bm{\hat{x}}_t, \\
&f(\bm{x}_t) + g(\bm{x}_t) \bm{u}_t \in \mathcal{X}_{t+1},~ \bm{u}_t \in \mathcal{U}_t
\end{align}
\label{eq:problem5}
\end{subequations}
where $c^{\mathrm{FOP}}(\bm{x}_t, \bm{u}_t)$ denotes the objective of the lower-level optimization problem, which we assume to be convex to support real-time decision-making. Importantly, $c^{\mathrm{FOP}}(\bm{x}_t, \bm{u}_t)$ is not required to be the same as the overarching task stage cost $c(\bm{x}_t,\bm{u}_t)$; instead, it can be viewed as a convex surrogate, while any mismatch, e.g., due to an intrinsically non-convex \(c(\bm{x}_t,\bm{u}_t)\), can be learned from data.

By leveraging intrinsic subgoals learned by the upper-level policy, the lower-level executor is able to make near-instantaneous, subgoal-conditioned decisions, which is particularly advantageous in time-sensitive or high-dimensional environments. However, the formulation of FOP still remains ambiguous in the bi-level framework, as discussed in Remark~\ref{rmk:2}.

\begin{remark}[Importance of proper lower-level formulations] \label{rmk:2}
 We highlight two considerations for Problem (\ref{eq:problem5}). First, the transformed desired next state $\bm{\hat{x}}_t$ produced by the upper-level RL may be infeasible in practice, e.g., planned trajectories in robotics tasks, which require the design of a cost function to penalize violations. Such a cost function must be aligned with the overarching objective of the decision-making problem to avoid sub-optimality. 
 Second, even when $\hat{x}$ is feasible, the formulation of the lower-level optimization still requires careful design to mitigate myopic decisions, particularly when the transformation matrix $C$ reduces the state dimension. Specifically, there may exist multiple actions $\bm{u}_t$ that satisfy the constraint $Cf(\bm{x}_t)+Cg(\bm{x}_t)\bm{u}_t=\hat{\bm{x}}_t$ but yield different next states under the true dynamics, which will influence future rewards. Therefore, it is important to properly design the lower-level problem to align the action selection with the overarching objective.

\end{remark}

In the next subsection, we will present an inverse optimization approach to inform the design of the lower-level optimization problem.

\subsection{Learning Problem Formulation: Inverse Optimization}

\subsubsection{General Framework}

We assume access to a set of expert data $\{\bm{x}_t, \bm{u}_t\}_{t\in\mathcal{T}_{e}}$ generated by an existing decision-making policy, such as Model Predictive Control (MPC) or other high-quality heuristics ensuring strict safety guarantees, where $\mathcal{T}_e$ represents the set of time steps at which the dataset is collected. 
We make three remarks regarding the expert dataset. First, as high-quality solution methods may not be suitable for making real-time decisions in practice, the expert data can be derived offline rather than from real-world operations. Second, for the same reason, the dataset is assumed to be small, which makes supervised learning methods such as imitation learning less suitable. Third, $\mathcal{T}_e$ is chosen to cover diverse operating conditions and does not have to consist of consecutive time steps.



The core of our approach is to recover latent structure encoded in the expert data. Specifically, we parameterize the lower-level optimization problem as 
\begin{equation}
\begin{aligned}
  \mathrm{FOP}_{\bm{\theta}}(\bm{x}_t,\bm{\hat{x}}_t) :=& \arg\min_{\bm{u}_t} \{ \tilde{c}(\bm{x}_t, \bm{u}_t)+ d_{\bm{\theta}}(\bm{x}_t,\bm{\hat{x}}_t, \bm{u}_t)
\\ &\mid \bm{a}_t + B_t \bm{u}_t = \bm{\hat{x}}_t,  \bm{u}_t \in \widehat{\mathcal{U}}_t\} \label{eq:parameterization}  
\end{aligned}
\end{equation}

where $\bm{a}_t:=Cf(\bm{x}_t)$, $B_t:=Cg(\bm{x}_t)$, and $\widehat{\mathcal{U}}_t:=\text{co}\left(\mathcal{U}_t\cap \{\bm{u}_t|f(\bm{x}_t)+g(\bm{x}_t)\bm{u}_t\in\mathcal{X}_{t+1}\}\right)$, where $\text{co}(\cdot)$ represents the convex hull. Specifically, the term $\tilde{c}(\bm{x}_t, \bm{u}_t)$ corresponds to the exact stage cost $c(\bm{x}_t, \bm{u}_t)$ when it is convex; otherwise, it serves as a convex surrogate to preserve computational tractability. Typical choices for $\tilde{c}(\cdot)$ include quadratic approximations or convex envelopes. 
The term $d_{\bm{\theta}}(\cdot)$ parameterized by parameter $\bm{\theta}$ is also assumed to be convex and will be properly designed to align the lower-level optimization problem to the objective of the overarching decision-making problem. 

We then convert the design of the lower-level optimization problem into estimating $\bm{\theta}$ from the expert dataset. 
Speficically, we aim to find \(\bm{\theta}\) via inverse optimization such that for each expert pair $(\bm{x}_t,\bm{u}_t)$, the action $\bm{u}_t$ is (approximately) optimal for \(\mathrm{FOP}_{\bm{\theta}}(\bm{x}_t,\bm{\hat{x}}_t)\) under some subgoal~$\bm{\hat{x}}_t$. Specifically, let $\mathcal{U}^{\text{opt}}(\bm{\theta},\bm{x}_t,\bm{\hat{x}}_t)$ denote the optimal solution set corresponding to parameter $\bm{\theta}$, state $\bm{x}_t$, and subgoal $\bm{\hat{x}}_t$. 
In practice, since the lower-level optimization is convex, $\mathcal{U}^{\text{opt}}(\bm{\theta},\bm{x}_t,\bm{\hat{x}}_t)$ can be characterized by the Karush–Kuhn–Tucker (KKT) conditions. 
Then, the inverse optimization problem can be formulated as 
\begin{equation}
\min_{\bm{\theta}} \left\{ \kappa h(\bm{\theta}) + \frac{1}{|\mathcal{T}_e|} \sum_{t\in \mathcal{T}_e} \ell\left(\bm{u}_t, \mathcal{U}^{\text{opt}}(\bm{\theta},\bm{x}_t,\bm{\hat{x}}_t)\right) \;\middle|\; \bm{\theta} \in \bm{\Theta} \right\}. \label{eq:IO}
\end{equation}
where the objective function is a weighted combination of two parts: (i) the sum of losses with each loss $\ell(\bm{u}_t, \mathcal{U}_t^{\text{opt}}(\bm{\theta}))$ indicating the deviation of expert action $\bm{u}_t$ from the optimal solution set $\mathcal{U}_t^{\text{opt}}(\bm{\theta})$, i.e., the sub-optimality of the expert action, and (ii) a user-defined, application-specific regularization term $h(\bm{\theta})$ representing prior information or user preference regarding $\bm{\theta}$, with $\gamma$ as the weight. This formulation follows a standard data-driven inverse optimization paradigm~\citep{chan2025inverse}.

\begin{remark}[Inverse optimization for lower-level formulations]
We make three remarks regarding the inverse optimization framework. First, we focus on learning the objective function of $\mathrm{FOP}_{\bm{\theta}}$ rather than its constraints, as constraints are typically dictated by physical requirements. Second, without loss of generality, we assume the constraint $\bm{a}_t+B_t\bm{u}_t=\bm{\hat{x}}_t$ is feasible. If this is not the case, we can relax it as $\bm{a}_t+B_t\bm{u}_t=\bm{\hat{x}}_t+\bm{\epsilon}_t$, and augment the action, corresponding matrix, and cost function as $\bm{\tilde{u}}_t=[\bm{u}_t;\bm{\epsilon}_t],~\tilde{B}_t=[B_t, I]$, and $d_{\bm{\theta}}(\bm{x}_t,\bm{\hat{x}}_t,\bm{\tilde{u}}_t)$, respectively, which yields a lower-level optimizer of the same form. 
\end{remark}

Without loss of generality, we assume $\widetilde{\mathcal{U}}_t$ is a polytope represented by $\widetilde{\mathcal{U}}_t = \{u\mid H_t\bm{u}\leq \bm{b}\}$. Under this assumption, the learning problem for $\theta$ can be formulated as
\begin{subequations}
    \begin{align}
  &\min_{\bm{\theta},\{\bm{w}_t,\bm{\lambda}_t,\bm{\epsilon}_t\}_{t\in\mathcal{T}_e}} \quad \sum_{t\in \mathcal{T}_e}\Vert\bm{\epsilon}_t\Vert_2^2 +  \kappa h(\bm{\theta})  \label{eq:KKT_1}\\ 
  \text{s.t.}& \bm{0} \in \partial_{\bm{u}} \left(\tilde{c}(\bm{x}_t, \bm{u}_t) + d_{\bm{\theta}}(\bm{x}_t,\bm{\hat{x}}_t, \bm{u}_t)\right) + \bm{\lambda}_t^TH_t + \bm{w}_t^TB_t,~\forall t\label{eq:KKT_2}\\
  & \bm{a}_t + B_t \bm{u}_t = \bm{\hat{x}}_t,\quad H_t\bm{u}\leq \bm{b},~\forall t\label{eq:KKT_3} \\
  & \bm{\lambda}_t \geq 0,~\forall t \label{eq:KKT_4}\\ 
  & \varphi(\bm{u}_t,\bm{\lambda}_t,\bm{w}_t) + \bm{\epsilon}_t = c(\bm{x}_t, \bm{u}_t) + d_{\bm{\theta}}(\bm{x}_t,\bm{\hat{x}}_t, \bm{u}_t),~\forall t\label{eq:KKT_5}
\end{align} \label{eq:KKT}
\end{subequations}
where $\partial_{\bm{u}}$ denotes the subdifferential of the cost function of $\mathrm{FOP}_{\bm{\theta}}$, which generalizes the gradient to allow for non-smooth objectives. The vectors $\bm{w}_t$ and $\bm{\lambda}_t$ are dual variables. The objective function in (\ref{eq:KKT_1}) follows the structure of Problem (\ref{eq:IO}), penalizing the sum of squared duality gaps while regularizing $\bm{\theta}$ via $h(\cdot)$. Constraints  (\ref{eq:KKT_2})-(\ref{eq:KKT_5}) are the KKT conditions, where (\ref{eq:KKT_5}) relaxes the strong duality condition with a duality gap $\bm{\epsilon}_t$. 

Concretely, we minimize the duality gap in Problem~(\ref{eq:KKT}) instead of the distance between given expert decisions and optimal solutions to the forward optimization problem. This aims to mitigate the computational burden associated with solving mathematically challenging problems, for instance, when Problem~(\ref{eq:KKT}) is non-convex, which will be analyzed in more detail in the following sections. 


\subsubsection{Special Case with Theoretical Analysis}



To illustrate how inverse optimization serves as a principled guideline for designing lower-level formulations, we consider a special class of decision-making problems with a linear stage cost function $ \bm{c}^\top \bm{u}_t$ and state-independent function $g(\cdot)$ (i.e., $B_t=B,\forall t$).
We show that under this setting, our proposed cost structure can achieve desirable properties such as inverse feasibility, forward stability, and computation efficiency. At the end of this section, we generalize our conclusions to a broader class of problems with quadratic stage cost functions. 

We make the following remarks on this special setting. First, this setting captures a broad range of applications in resource allocation, logistics, and energy systems, where strict constraint satisfaction and fast computation are critical. 
Second, it is important to note that despite the linearity of the stage cost, the original multi-stage decision-making problem~(\ref{eq:problem1}) can still be challenging to solve in real-world operations due to (i) the potentially nonlinear structures of $f(\cdot)$ and $g(\cdot)$ and (ii) potentially high state and action dimensions. 


In this special setting, we aim to derive an FOP formulation such that the inverse optimization is always feasible. This is important, as it can allow us to utilize any type of expert data. 
In this case, we formulate the estimation of $\bm{\theta}$ as  
\begin{equation}
\min_{\bm{\theta}} \left\{ \varphi(\bm{\theta}) \;\middle|\; \bm{\theta} \in \bm{\theta}_t^{\text{inv}}(\bm{u}_t) \; \forall t \in \mathcal{T}_e,\; \bm{\theta} \in \bm{\Theta} \right\}.
\label{eq:problem8}
\end{equation}
where $\bm{\theta}_t^{\text{inv}}(\bm{u}_t) := \left\{ \bm{\theta} \;\middle|\; \bm{u}_t \in \mathcal{U}_t^{\text{opt}}(\bm{\theta}) \right\}$ is the inverse feasible set, i.e., the set of parameter values under which the expert action $\bm{u}_t$ belongs to the optimal solution set $\mathcal{U}_t^{\text{opt}}(\bm{\theta})$ of the lower-level problem $\mathrm{FOP}_{\bm{\theta}}(\bm{x}_t, \bm{\hat{x}}_t)$. The function $\varphi(\bm{\theta})$ denotes a user-defined, application-specific objective that resolves indeterminacy by selecting among multiple feasible parameters. 

As stated in Proposition~\ref{prp:1},
without a properly designed cost structure $d_{\bm{\theta}}(\cdot)$, the lower-level formulation lacks the expressivity to rationalize expert behaviors, which means there may exist no feasible upper-level subgoal $\hat{x}_t$ that renders the observed expert actions optimal.



\begin{proposition}[]\label{prp:1}
    The inverse feasible set for Problem (\ref{eq:problem5})\[
\bm{\hat{\mathcal{X}}}_t^{\text{inv}}(\bm{u}_t) := \left\{ \bm{\hat{x}_t} \;\middle|\; \bm{u}_t \in \mathcal{U}_t^{\text{opt}}(\bm{\hat{x}_t}) \right\}.
\] is not guaranteed to be non-empty, where $\mathcal{U}_t^{\text{opt}}(\bm{\hat{x}_t})$ is defined by the KKT conditions
 of Problem (\ref{eq:problem5}).
\end{proposition} 

To address the sub-optimality, we propose an inverse-optimization–guided design procedure that leverages criteria such as inverse feasibility and forward stability to inform the lower-level formulation.

\vspace{0.2em}\noindent \textbf{Stage 1: Validating optimality of expert decisions.} We first construct a preliminary formulation of $\mathrm{FOP}_{\bm{\theta}}$ such that the historical decisions $\{\bm{u}_t\}_{t=1}^T$ are possible to be optimal under the observed states $\{\bm{x}_t\}_{t=1}^T$. We achieve this goal by validating the feasibility of the inverse problem. 

We select the terminal cost $d_{\bm{\theta}}(\cdot)$ represented by the summation of ReLU-based regularization terms, which demonstrates high efficiency despite its simplicity. Propositions~\ref{pro:pro3.6} and~\ref{pro:pro3.7} show the motivation to design such cost structure from both algebraic and geometric perspectives respectively.
  \begin{equation}
  \label{eq:regu}
\begin{aligned}
\text{FOP}(\bm{x}_t, \bm{\hat{x}}_t) :=  \arg\min_{\bm{u}_t} \{ &\bm{c}^\top \bm{u}_t + \sum_{k=1}^K \max\{\bm{\theta}^T_k\bm{u}_t-\nu_k, 0\} \\&\big|  \bm{a}_t + B \bm{u}_t = \bm{\hat{x}}_t, \bm{u}_t \in \widetilde{\mathcal{U}}_t \}
\end{aligned}
\end{equation} 
\begin{proposition}
    The inverse feasible set for Problem (\ref{eq:regu}), $\bm{\theta}_t^{*\text{inv}}(\bm{u}_t) := \left\{ \bm{\theta} \;\middle|\; \bm{u}_t \in \mathcal{U}_t^{*\text{opt}}(\bm{\theta}) \right\}$, is always non-empty, where $\mathcal{U}_t^{*\text{opt}}(\bm{\theta})$ is defined by KKT conditions of Problem (\ref{eq:regu}).
    \label{pro:pro3.6}
\end{proposition}

\begin{proposition}
    By appropriately enlarging the decision space with auxiliary variables in Problem (\ref{eq:regu}), any given set of expert decisions can be made optimal in the lifted space.
    \label{pro:pro3.7}
\end{proposition}

\vspace{0.2em}\noindent \textbf{Stage 2: Ensuring forward stability.} 
    Although Problem (\ref{eq:regu}) has proven to make expert data optimal, i.e., the inverse feasibility is guaranteed, another critical issue in inverse optimization, named forward stability, proposed by \citet{shahmoradi2022quantile}, is not guaranteed in the case of a linear program. The definition of forward stability is defined as follows.
\begin{definition}[Forward Instability \citep{shahmoradi2022quantile}]
\label{def:forward-instability}
Given a set of expert observations $\hat{\mathcal{U}}$, the \emph{forward instability} of an inverse solution $\hat{\bm{\theta}} \in \bm{\theta}^{*\text{inv}}(\hat{\mathcal{U}})$ is defined as
\begin{align}
\max_{u \in \mathcal{U}^{*opt}(\hat{\bm{\theta}})} \left\{ d(\hat{\mathcal{U}}, u) \right\},
\end{align}
where $\mathcal{U}^{*opt}(\hat{\bm{\theta}})$ denotes the set of forward optimal solutions corresponding to $\hat{\bm{\theta}}$. This value quantifies the worst-case distance between a forward solution $u$ induced by $\hat{\bm{\theta}}$ and the expert data $\hat{\mathcal{U}}$, measuring how unstable the inverse solution $\hat{\bm{\theta}}$ can be.
\end{definition}

 To solve this issue, we include small quadratic terms into objective function to ensure strong convexity of objective function, such that forward stability is improved and the expert decisions are approximately optimal at the same time. Overall, we formulate the forward optimization problem as follows.
\begin{equation}
\begin{aligned}
\text{FOP}(\bm{\hat{x}}_t, \bm{x}_t) := \ & \arg\min_{\bm{u}_t}\{\bm{c}^\top \bm{u}_t + \sum_{k=1}^K \max\{\bm{\theta}^\top_k \bm{u}_t - \nu_k, 0\} \\
&+ l \|\bm{u}_t\|_2^2\ \big|\ \bm{a}_t + B \bm{u}_t = \bm{\hat{x}}_t, \bm{u}_t \in \widetilde{\mathcal{U}}_t\}
\end{aligned}
\label{eq:fop}
\end{equation}
We propose the following inverse optimization formulation, where dual variables are denoted by $z_{kt},y_{kt}, \bm{w}_t, \bm{\lambda}_t$ and $\tau_{kt}$ is the auxiliary variable. $g(\bm{\lambda}_t, \bm{z}_t, \bm{y}_t, \bm{w}_t)$ is the dual objective function with explicit form in this problem. To keep the formulation concise, we omit the explicit form of this function. 

\begin{equation}\label{eq:problem10}
\begin{aligned}
&\min_{\substack{z_{kt}, y_{kt}, \bm{\lambda}_t, \bm{w}_t, \\ \bm{\epsilon}_t, \bm{\theta}_k,\nu_k}} \quad \sum_{t \in \mathcal{T}_e} \|\bm{\epsilon}_t\|_2^2 + \rho \sum_k \|\bm{\theta}_k\|_2^2 \\
\text{s.t.} \quad 
& 1 - z_{kt} - y_{kt} = 0, \quad k = 1,\dots,K,\ t \in \mathcal{T}_e \\
& \bm{c}^\top + \bm{\lambda}_t^\top H + \bm{w}_t^\top B \\
& \qquad + \sum_k z_{kt}\bm{\theta}_k^\top + 2l\bm{u}_t = 0, \quad t \in \mathcal{T}_e \\
& g(\bm{\lambda}_t, \bm{z}_t, \bm{y}_t, \bm{w}_t) + \epsilon_t = \bm{c}^\top \bm{u}_t \\
& \qquad + \sum_{k=1}^K \tau_{kt} + l\bm{u}_t^\top \bm{u}_t, \quad t \in \mathcal{T}_e \\
& \tau_{kt} \geq \bm{\theta}^\top_k \bm{u}_t - \nu_k, \quad k = 1,\dots,K,\ t \in \mathcal{T}_e \\
& z_{kt}, y_{kt}, \bm{\lambda}_t, \tau_{kt} \geq 0, \quad k = 1,\dots,K,\ t \in \mathcal{T}_e
\end{aligned}
\end{equation}

The problem is non-convex including bi-linear terms in constraints which poses significant challenges for computational tractability. We minimize duality gap together with regularizing norm of $\bm{\theta}_k$ for robustness. For solving Problem (\ref{eq:problem10}), it can be reformulated as a Mixed Integer Program and solved using the spatial branch-and-bound algorithm within a reasonable time \citep{smith1999symbolic}. By minimizing the duality gap instead of the distance between optimal solutions and given expert solutions, the number of bilinear terms is reduced from $\mathcal{O}(|\mathcal{A}|K+|\mathcal{A}|^2)$ to $\mathcal{O}(|\mathcal{A}|K)$. Thus, the number of constraints introduced by the construction of the McCormick envelope can be significantly reduced. 

What's more, Proposition~\ref{pro:pro3.9} demonstrates that the distance between the input and optimal solutions remains bounded. By introducing a sufficient number of ReLU-based terms, we can constrain the values of $\epsilon_0$ within a small range, thereby controlling the upper bound of the distance. 



\begin{proposition}
    Problem (\ref{eq:problem10}) yields optimal solutions $\bm{u}_t^*$ satisfying \( |f(\bm{u}_t^*) - f(\bm{u}_t)| \leq \epsilon_0 \),  s.t.
\begin{align}
\|\bm{u}_t^* - \bm{u}_t\| \leq \sqrt{\frac{\epsilon_0}{l}}
\end{align}
where $\epsilon_0=\max_t\bm{\epsilon}_t^*$ and $\bm{u}_t$ is the given expert decision in time step $t$.
\label{pro:pro3.9}
\end{proposition}

For the broader class of control problems  with quadratic objective functions, we can still infer parameters as shown in Problem (\ref {eq:problem10}). Proposition~\ref{pro:pro3.9} can be extended as shown below. 
\begin{corollary}
    When the objective function is given by $\bm{u}^TR\bm{u}+\bm{x}^TQ\bm{x}$ with the assumption $R\succ 0$, the upper bound will be achieved by $\sqrt{\frac{\lambda_{min}(R)}{l}}$, where $\lambda_{min}(R)$ denotes the minimum eigenvalue of $R$.
\end{corollary}


\section{Case Study}
In this section, we evaluate the proposed framework through three case studies from different fields: (1) autonomous vehicle rebalancing; (2) inventory management in multi-warehouse multi-retailer supply chain, and (3) mobile robot navigation. The results highlight its ability to enhance decision quality and provide interpretable solutions in dynamic environments. The details of the three environments are given in Appendix~\ref{app:case studies}.

\subsection{Experimental Setup}
\textbf{Data Collection:} We generate expert demonstrations by simulating a full-horizon MPC controller over a finite horizon \( T \). The dataset consists of state-action pairs collected offline, designed to cover diverse operating conditions for each distinct task. Crucially, we utilize only a handful of these demonstrations to infer the formulation of the lower-level optimization policy via inverse optimization

\textbf{RL Implementation:} We employ A2C with Graph Neural Networks in the hierarchical framework for the AV and Supply Chain tasks, and Soft Actor-Critic (SAC) for the Mobile Robot task. Detailed network architectures and hyperparameters are provided in Appendix~\ref{app:c11}.

\textbf{Oracle and Baselines:} 
We use \textbf{MPC} with long-time horizon as an Oracle, and baselines include:
\begin{itemize}
    \item \textbf{Bi-level-original}: RL-OC framework where lower-level optimization policy is unchanged with original single-step cost \citep{gammelli2023graph}. 
    \item \textbf{End-to-end RL}: End-to-end reinforcement learning that trains a single model to directly map inputs to the final control actions.
    \item \textbf{Inverse MPC}: Identify the optimal cost function that minimizes the discrepancy between the actual trajectory and expert demonstration \citep{zhang2024inverse}.
    \item \textbf{S-type Heuristic}: Given the widespread use of established heuristics such as $(S, s)$ and $(Q, r)$ policies in inventory management, we specifically include a representative heuristic baseline to evaluate our framework's performance.
    \item \textbf{Value Approximation}: Quadratic terminal cost estimation via value function~\citep{abdufattokhov2021learning}, whose performance will be analyzed in the Appendix~\ref{app:c31}.
\end{itemize}


\subsection{Model Evaluation}


We compare the RL-OC framework equipped with a learned lower-level policy against variant baselines. Table~\ref{tab:2} and~\ref{tab:combine} summarize the comparative results. First, our hierarchical RL-OC framework achieves a marked improvement, which highlights the effectiveness applying RL-OC framework to mitigate the curse of dimensionality by compressing the original action space into lower-dimensional subgoals. Building on the hierarchical structure, our Bi-level-learned approach improves upon the Bi-level-original baseline by employing a refined lower-level formulation. Additionally, the framework significantly outperforms Inverse MPC baselines, confirming that our ReLU-based cost structure informed by inverse optimization provides a more principled design than the standard quadratic forms used in prior works. 

\begin{table}[h!]
\centering
\caption{Performance Comparison on Supply Chain Inventory Management.}
\resizebox{\columnwidth}{!}{
\label{tab:im_results}
\begin{tabular}{@{}lll@{}}
\toprule
\textbf{Method} & \textbf{Reward} & \textbf{Served Demand} \\
\midrule
MPC & 11335 ($\pm$34.3) & 1337 ($\pm$5.61) \\
\arrayrulecolor{black}
\hdashline[3pt/2pt] 
\noalign{\vskip 2pt} 
End-to-end RL & 2764 ($\pm$90.9) & 476 ($\pm$18.4) \\
Inverse MPC & 5608 ($\pm$31.8) & 643 ($\pm$7.04)\\
S-type Heuristic & 2709 ($\pm$ 195.6) & 601 ($\pm$ 12.5)\\
Bi-level-original & 7491 ($\pm$26.5) & 778 ($\pm$2.62) \\
Bi-level-learned (Ours) & \textbf{9442} ($\pm$42.3) & \textbf{930} ($\pm$1.17) \\
\bottomrule
\end{tabular}}
\label{tab:2}
\end{table}
\begin{table*}[t!]
  \centering
  \caption{Performance Comparison on AV Rebalancing and Mobile Robot Navigation tasks}
  \label{tab:integrated_results}
  \small
  \begin{tabular}{lccccccc}
    \toprule
    \multirow{2.5}{*}{\textbf{Method}} & \multicolumn{2}{c}{\textbf{AV Rebalancing}} & & \multicolumn{3}{c}{\textbf{Mobile Robot Navigation}} \\
    \cmidrule(lr){2-3} \cmidrule(lr){5-7}
    & \textbf{Reward $\uparrow$} & \textbf{Served Demand $\uparrow$} & & \textbf{Time (s) $\downarrow$} & \textbf{Length (m) $\downarrow$} & \textbf{Energy (J) $\downarrow$} \\
    \midrule
    MPC (Oracle) & 35725 ($\pm$41.6) & 3203 ($\pm$3.07) & & 3.50 & 3.50 & 4.81 \\
    \midrule
    End-to-end RL & 20989 ($\pm$213.5) & 2334.8 ($\pm$14.9) & & $7.60 \pm 0.11$ & $4.29 \pm 0.02$ & $6.78 \pm 0.23$ \\
    Inverse MPC & 32426 ($\pm$401.2) & 2996 ($\pm$28.6) & & * & * & * \\
    Bi-level-original & 33406 ($\pm$128.1) & 2993 ($\pm$8.06) & & $9.14 \pm 0.08$ & $4.35 \pm 0.02$ & $6.20 \pm 0.13$ \\
    \textbf{Bi-level-learned (Ours)} & \textbf{35019} ($\pm$53.0) & \textbf{3141} ($\pm$6.14) & & $\mathbf{4.40 \pm 0.10}$ & $\mathbf{4.14 \pm 0.07}$ & $\mathbf{1.52 \pm 0.06}$ \\
    \bottomrule
  \end{tabular}
  \begin{flushleft}
    \footnotesize{Note: Upward ($\uparrow$) and downward ($\downarrow$) arrows indicate that higher or lower values are preferred, respectively. Asterisks (*) denote collisions where the agent fails to reach the goal. MPC serves as the Oracle.}
  \end{flushleft}
  \label{tab:combine}
\end{table*}
The qualitative advantages of our framework are further illustrated in Figure~\ref{fig:mobile} below and Figure~\ref{fig:two_tasks} in Appendix C.2.2 , which compare the system trajectories across different tasks. By leveraging inverse optimization, the proposed framework effectively recovers the latent behavioral patterns in expert demonstrations. As a result, the trajectories generated by our Bi-level-learned method exhibit significantly higher alignment with the expert (MPC) benchmarks compared to the Bi-level-original baseline.

\begin{figure*}[h]
    \centering
    \begin{subfigure}[t]{0.33\linewidth}
        \centering
        \includegraphics[width=\linewidth]{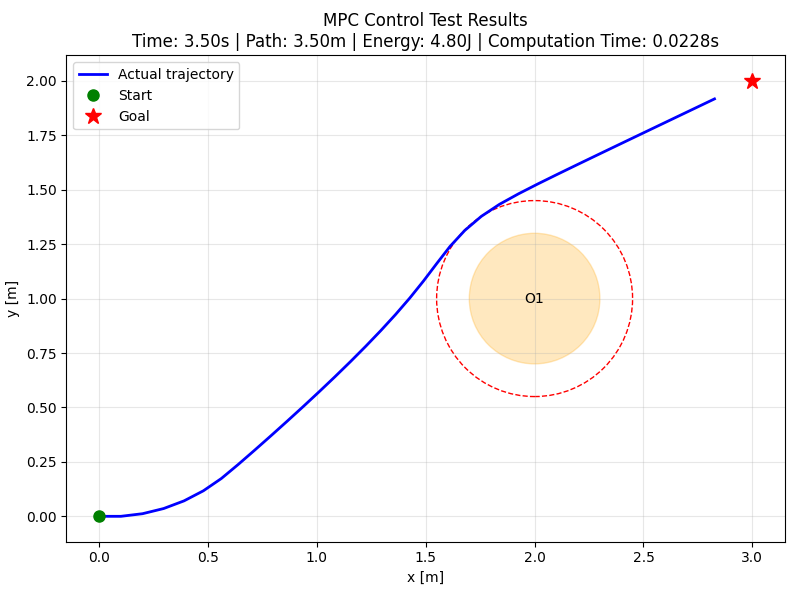}
        \caption{Expert demonstrations (MPC)}
        \label{fig:mobile_mpc}
    \end{subfigure}\hfill
    \begin{subfigure}[t]{0.33\linewidth}
        \centering
        \includegraphics[width=\linewidth]{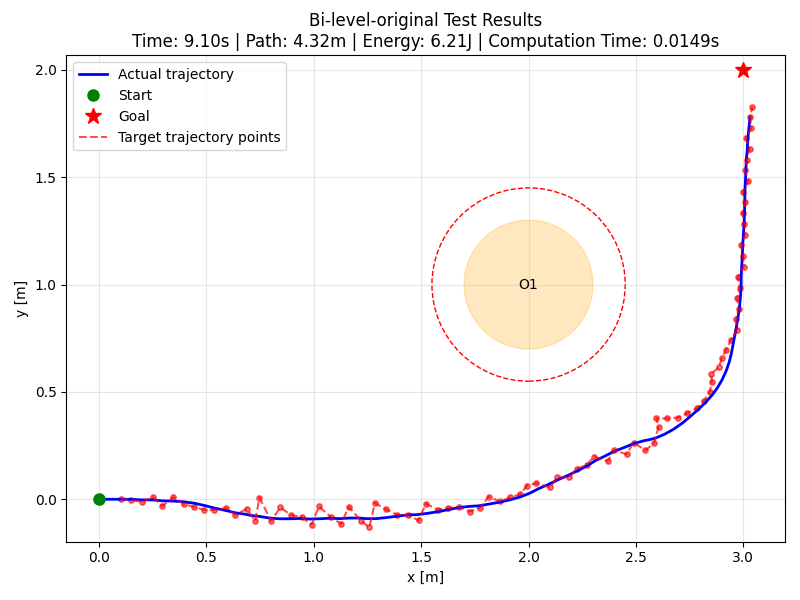}
        \caption{Bi-level-original}
        \label{fig:mobile_ori}
    \end{subfigure}\hfill
    \begin{subfigure}[t]{0.33\linewidth}
        \centering
        \includegraphics[width=\linewidth]{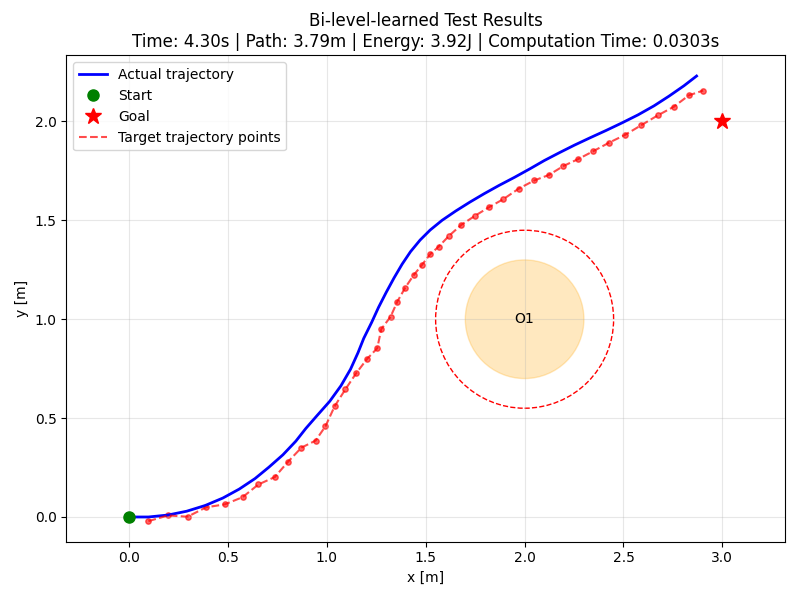}
        \caption{Bi-level-learned}
        \label{fig:mobile_learned}
    \end{subfigure}
    \caption{Trajectory comparison for the Mobile Robot Navigation task across different methods}
    \label{fig:mobile}
\end{figure*}

Moreover, compared with MPC, the bi-level framework shortens the planning horizon from T steps to one step, which significantly reduces decision-making time. We compare the time spent to make decisions for one step in Table~\ref{tab:time_across_tasks}. It shows a significant reduction in runtime which indicates that our method can significantly improve computational efficiency without substantially compromising solution quality.
\begin{table}[h!]
  \centering
  \caption{Comparison of Computational Time Across Scenarios}
  \label{tab:time_across_tasks}
  \begin{tabular}{lcc}
    \toprule
    \multirow{2}{*}{\textbf{Scenarios}} & \multicolumn{2}{c}{\textbf{Runtime (s)}} \\ \cmidrule(lr){2-3}
    & \textbf{Our Method} & \textbf{MPC} \\
    \midrule
    AV Rebalancing & \textbf{0.025} & 1.44  \\
    Inventory Management           & \textbf{0.034} & 0.082  \\
    Mobile Robot Navigation  & \textbf{0.0187} & 0.0275  \\
    \bottomrule
  \end{tabular}
\end{table}
These results highlight the effectiveness of optimizing the lower-level problem structure in enhancing overall system performance. Besides, scalability analysis, sensitivity analysis, interpretations of learned cost structure and motivations of proposed methods are provided and discussed in the Appendix ~\ref{app:experiment}.

\section{Limitations and Future Work}
While the developed hierarchical RL-OC scheme provides an effective solution for real-time, safety-critical control tasks, it presents certain limitations that motivate our future research directions.

First, our methodology is primarily evaluated on tasks with structured state representations. To broaden its real-world applicability, we intend to scale the proposed framework to larger-scale systems with complex dynamics, and to expand its scope to applications with unstructured inputs. A key advantage of our hierarchical RL-OC architecture is its potential to bridge high-dimensional perception with structured control, where the upper-level RL learns to extract latent representations and generate low-dimensional subgoals, while the lower-level optimization then strictly enforces physical constraints while optimizing for consistency with these learned representations and long-term objectives. Second, the quality of the solutions to the inverse optimization depends on the quality of the provided expert demonstrations. Although our empirical sensitivity analysis demonstrates reasonable robustness to data noise, learning from highly sub-optimal data remains challenging. We aim to conduct a more rigorous investigation into handling noisy or sub-optimal expert demonstrations. A natural extension is to develop noise-aware inverse formulation.

\section{Conclusion}
In this work, we propose a hierarchical RL-OC framework that addresses key challenges in real-time, safety-critical decision-making scenarios, where the upper-level policy generates subgoals that abstract planning objectives, while the lower-level controller executes these subgoals through a structured optimization problem. A central contribution of this work is the principled construction of the lower-level optimization formulation via inverse optimization, rather than relying on heuristic design. By recovering the latent cost structure from expert demonstrations, we ensure that the lower-level controller is aligned with the overarching objectives. Empirical evaluations across diverse domains demonstrate that our framework not only outperforms strong baselines in decision quality but also provides interpretable control policy. This work highlights a promising direction for combining learning-based goal abstraction with structured optimization. In particular, our study positions inverse optimization as an effective bridge between reinforcement learning and optimal control.

\section*{Acknowledgements}This research was supported by the Singapore Ministry of Education (MOE) under its Academic Research Fund Tier 1 (A-8001183-00-00).
\section*{Impact Statement}
This paper presents a methodological advance in integrating reinforcement learning and optimal control by leveraging inverse optimization to construct structured policies. There are many potential societal consequences of our work, particularly in improving the efficiency and safety of cyber-physical systems, including transportation, power grids, and robotics. Overall, we don't foresee any specific societal consequences of this work that necessitate a separate discussion.






\bibliography{reference}
\bibliographystyle{icml2026}

\newpage
\appendix
\onecolumn

\section{Proof of Propositions}
\subsection{Proposition 3.5}
\label{app:proof_proposition1}
    The inverse feasible set for Problem (\ref{eq:problem5})\[
\bm{\hat{\mathcal{X}}}_t^{\text{inv}}(\bm{u}_t) := \left\{ \bm{\hat{x}_t} \;\middle|\; \bm{u}_t \in \mathcal{U}_t^{\text{opt}}(\bm{\hat{x}_t}) \right\}.
\] is not guaranteed to be non-empty, where $\mathcal{U}_t^{\text{opt}}(\bm{\hat{x}_t})$ is defined by the KKT conditions
 of Problem (\ref{eq:problem5}).
\begin{proof}
We explore whether there exists $\bm{\theta}$ belonging to the inverse feasible set for Problem (\ref{eq:problem5}), which is making $\bm{u}_t$ to satisfy all KKT conditions.
\begin{itemize}
    \item KKT conditions for Problem (\ref{eq:problem5}):
    \begin{enumerate}
        \item Stationary Conditions: 
        \begin{equation*}
            \begin{aligned}
        \bm{c}^T+\bm{\lambda}_t^TH+\bm{w}_t^TB = 0\\
            \end{aligned}
        \end{equation*}
        \item Complementary Slackness:
        \begin{equation*}
            \begin{aligned}
                \bm{\lambda}_t^T(H\bm{u}_t-h)&=0
            \end{aligned}   
        \end{equation*}
        \item Dual Feasibility and Primal Feasibility:
        \begin{equation*}
        \begin{aligned}
            \bm{\lambda}_t\geq 0\\
            H\bm{u}_t\leq h
        \end{aligned}         
        \end{equation*}      
    \end{enumerate}
\end{itemize}
We assume $H\bm{u}_t\leq h$ is not active, which means $\bm{\lambda}_t = 0$. The system $B^T\bm{w}_t = -c$ has no feasible solution if the vector $-c$ does not lie in the column space of $B^T$, for example, when the system is overdetermined. Specifically, the overdetermined nature of the system arises from our design, where the sub-goal space is intentionally compressed relative to the action space.
\end{proof}

\subsection{Proposition 3.6}
    The inverse feasible set for Problem (\ref{eq:regu}) \[
\bm{\theta}_t^{*\text{inv}}(\bm{u}_t) := \left\{ \bm{\theta} \;\middle|\; \bm{u}_t \in \mathcal{U}_t^{*\text{opt}}(\bm{\theta}) \right\}.
\] is always non-empty, where $\mathcal{U}_t^{*\text{opt}}(\bm{\theta})$ is defined by KKT optimality conditions of Problem (\ref{eq:regu}).

\label{app:proof_proposition2}
\begin{proof}
We explore whether there exists $\bm{\theta}$ belonging to the inverse feasible set for Problem (\ref{eq:regu}), which is making $\bm{u}_t$ to satisfy all KKT conditions.
\begin{itemize}
    \item KKT conditions for Problem \ref{eq:regu}:
    \begin{enumerate}
        \item Stationary Conditions: 
        \begin{equation*}
            \begin{aligned}
        \bm{c}^T+\sum_{k=1}^Kz_{kt}\bm{\theta}_k^T+\bm{\lambda}_t^TH+\bm{w}_t^TB = 0,\quad for\ t \in \mathcal{T}_e\\
        1-z_{kt}-y_{kt} = 0,\quad for\ k=1,...K\ and\ t \in \mathcal{T}_e
            \end{aligned}
        \end{equation*}
        \item Complementary Slackness:
        \begin{equation*}
            \begin{aligned}
                z_{kt}(\tau_{kt}-\bm{\theta}_k^T\bm{u}_t+\nu_k) &= 0,\quad for\ k=1,...K\ and\ t \in \mathcal{T}_e\\
                y_{kt}\tau_{kt}&=0,\quad for\ k=1,...K\ and\ t \in \mathcal{T}_e\\
                \bm{\lambda}_t^T(H\bm{u}_t-h)&=0,\quad for\ t \in \mathcal{T}_e
            \end{aligned}   
        \end{equation*}
        \item Dual Feasibility and Primal Feasibility:
        \begin{equation*}
        \begin{aligned}
            z_{kt},\ y_{kt},\ \bm{\lambda}_t\geq 0,\quad for\ k=1,...K\ and\ t \in \mathcal{T}_e\\
            H\bm{u}_t\leq h,\quad for\ t \in \mathcal{T}_e\\
            \tau_{kt}\geq \bm{\theta}_k^T\bm{u}_t-\nu_k, \tau_{kt}\geq 0,\quad for\ k=1,...K\ and\ t \in \mathcal{T}_e
        \end{aligned}         
        \end{equation*}      
    \end{enumerate}
\end{itemize}
There exist trivial feasible solutions to satisfy all KKT conditions:\\
We assume $z_{kt}=1$, $y_{kt}=0$ and $\bm{\lambda}_t = \bm{0}$. Then the set of KKT conditions are transformed into 
\begin{equation}
    \begin{aligned}
    &\bm{c}^T+\sum_{k=1}^K\bm{\theta}_k^T+\bm{w}_t^TB = 0,\quad for\ t \in \mathcal{T}_e\\
        &\tau_{kt}= \bm{\theta}_k^T\bm{u}_t-\nu_k,\ \tau_{kt}\geq 0,\quad for\ k=1,...K\ and\ t \in \mathcal{T}_e
    \end{aligned}
\end{equation}
We must find $\{\bm{\theta}_k\}_{k=1}^K$ to satisfy $\bm{c}^T+\sum_{k=1}^K\bm{\theta}_k^T+\bm{w}_t^TB = 0$ and let $\nu_k=\min_{t=1}^T\{\bm{\theta}_k^T\bm{u}_t\}$
\end{proof}

\subsection{Proposition 3.7}
    By appropriately enlarging the decision space with auxiliary variables in Problem (\ref{eq:regu}), any given set of 
expert decisions can be made optimal in the lifted space.
\begin{proof}
Consider a collection of expert decisions $\{u_i\}_{i=1}^N$ with $u_i \in \mathbb{R}^n$. Problem~(\ref{eq:regu}) can be reformulated as a linear program by introducing auxiliary variables $\{t_k\}_{k=1}^K$, which effectively enlarges the decision space from $\mathbb{R}^n$ to $\mathbb{R}^{n+K}$. For clarity, let us focus on the case $K=1$. By adding the constraints 
\begin{align}
t_k \;\geq\; \bm{\theta}_k^\top u_i - \nu,
\end{align}
We can at least guarantee that all expert decisions $\{u_i\}_{i=1}^N$ lie on the same $n$-dimensional face of the feasible polytope in the augmented space by making the new added constraint active. This face can be made parallel to the hyperplane defined by the objective function by adjusting the new added constraint. 

This situation, however, represents an extreme case. In practice, the set $\{u_i\}_{i=1}^N$ may lie within a face of dimension strictly smaller than $n$. To illustrate, suppose $u_i \in \mathbb{R}^2$. If all expert decisions lie on a one-dimensional face (an edge) of the feasible polytope, then they remain on a corresponding one-dimensional face in the lifted $\mathbb{R}^3$ space. Assume that this edge arises as the intersection of two adjacent facets, whose supporting hyperplanes have normal vectors $\vec{n}_1$ and $\vec{n}_2$, respectively. The direction vector of the edge is then given by 
\begin{align}
\vec{d}_{\text{edge}} = \vec{n}_1 \times \vec{n}_2,
\end{align}
since it is simultaneously orthogonal to both $\vec{n}_1$ and $\vec{n}_2$. 

For the expert decisions to be optimal, it suffices that the objective hyperplane be parallel to this edge. Equivalently, the edge direction $\vec{d}_{\text{edge}}$ must be orthogonal to the objective normal $\vec{n}_{\text{obj}}$, i.e.,
\begin{align}
\vec{d}_{\text{edge}} \cdot \vec{n}_{\text{obj}} = 0.
\end{align}
This condition establishes the required relationship between the optimal geometry and the objective orientation. Thus, we can solve the equation for feasible parameters $\theta$ in the terminal cost. 
\end{proof}

\subsection{Proposition 3.9}
\label{app:proof_proposition5}
Problem (\ref{eq:problem10}) yields optimal solutions $\bm{u}_t^*$ satisfying \( |f(\bm{u}_t^*) - f(\bm{u}_t)| \leq \epsilon_0 \),  s.t.
\begin{align}
\|\bm{u}_t^* - \bm{u}_t\| \leq \sqrt{\frac{\epsilon_0}{l}}
\end{align}

\begin{proof}
First, we show that the feasible region of Problem (\ref{eq:problem10}) is non-empty.
Since the strong duality constraint has been relaxed, the main task reduces to verifying the stationarity condition
\begin{align}
\bm{c}^T+\bm{\lambda}_t^TH+\bm{w}_t^TB+\sum_k z_{kt}\bm{\theta}_k^T+2l \bm{u}_t=0, \qquad t=1,\ldots,T.
\end{align}
For the $i$-th expert decision, suppose there are $n_i$ active inequality constraints. Each such active constraint introduces one additional degree of freedom in the stationarity condition. Together with the $N$ equality constraints, the total number of effective constraints across $T$ time periods is 
\begin{align}
\sum_{i=1}^T n_i + NT.
\end{align}
Since the decision variable $c$ lies in $\mathbb{R}^M$, the stationarity condition can be satisfied provided that the number of auxiliary terms is large enough to compensate for the remaining degrees of freedom. More precisely, the required number of ReLU terms is bounded above by
\begin{align}
\max\!\left\{0,\, T-\frac{\sum_{i=1}^T n_i + NT}{M}\right\}.
\end{align}
This establishes the feasibility of the relaxed problem. Next, we derive the upper bound of distance.

Based on strong convexity of $f(u)$, we have 
\begin{equation}
f(\bm{u}_t) \geq f(\bm{u}_t^*) + \nabla f(\bm{u}_t^*)^\top (\bm{u}_t - \bm{u}_t^*) + \frac{\mu}{2} \|\bm{u}_t - \bm{u}_t^*\|^2
\end{equation}
where $\nabla^2 f(\bm{u}_t)\geq \mu I>0$. 

By stationary condition $\nabla f(\bm{u}_t^*) + \lambda_t^{*T}H + w_t^{*T}B = 0$ and complementary slackness condition $\lambda^{*T}_t(Hu_t^*-h_t)=0$,
\begin{align*}
        \nabla f(\bm{u}_t^*)^\top (\bm{u}_t - \bm{u}_t^*) =& (-\bm{\lambda}^{*T}_tH - \bm{w_t}^{*T}B)(\bm{u}_t - \bm{u}_t^*) = -\bm{\lambda_t}^{*T}H(\bm{u}_t - \bm{u}_t^*)\\
        =&\bm{\lambda}_t^{*T}(\bm{h_t}-H\bm{u}_t)\\
        \geq& 0
\end{align*}
Thus, 
\begin{equation}
f(\bm{u}_t) - f(\bm{u}_t^*) \geq \frac{\mu}{2} \|\bm{u}_t - \bm{u}_t^*\|^2
\end{equation}
The dual objective function $g(\bm{\lambda}_t, \bm{z}_t, \bm{y}_t, \bm{w}_t) \leq f(\bm{u}_t^*) \leq f(\bm{u}_t)$ and we have  
    $f(\bm{u}_t) - g(\bm{\lambda}_t, \bm{z}_t, \bm{y}_t, \bm{w}_t) \leq \epsilon_0$ by constraining strong duality gap, s.t.
\begin{equation}
    f(\bm{u}_t) - f(\bm{u}_t^*)\leq \epsilon_0
\end{equation}
Thus, we have
\begin{equation}
    \|\bm{u}_t - \bm{u}_t^*\|^2 \leq \epsilon_0*\frac{2}{\mu} = \frac{\epsilon_0}{l}
\end{equation}
    \label{lem:lemma1}
\end{proof}

\section{Details of Case Studies}
\label{app:case studies}
We consider three simulation scenarios described as follows:

\noindent \textbf{Autonomous Vehicle Rebalancing.} Autonomous Mobility-on-Demand (AMoD) systems are an evolving mode of transportation in which a centrally coordinated fleet of self-driving vehicles dynamically serves travel requests. In real-world systems, the effectiveness of rebalancing strategies is central to the overall system performance, with sub-optimal strategies potentially exacerbating congestion through unnecessary trips or increased passenger waiting time. The control of these systems is typically formulated as a large network optimization problem. This framework comprises two stages: (1) determining the desired distribution of idle vehicles $\hat{\bm{q}}_t$ through the use of the learned policy $\pi_\phi(\hat{\bm{q}}_t|s_t)$ by reinforcement learning, (2) converting this distribution to a passenger flow $g_{ij}^t$ and rebalancing flow $f_{ij}^t$ by solving a linear control problem. 

\noindent \textbf{Supply Chain Inventory Management.} In the supply chain inventory management task, we aim to determine the optimal ordering and distribution strategies across a network of interconnected multiple warehouses and multiple retail stores to satisfy customer demand while minimizing storage, production and transportation costs. We choose upper-level goals as (i) desired production in warehouse nodes $\hat{w}_i^t$ and (ii) desired inventory in store nodes $\hat{q}_i^t$. And the lower-level optimization module determines the amount of commodities $\bm{w}_i^t$ to order in each warehouse, and the shipping flow $f_{ij}^t$ from warehouses to stores. 

\noindent \textbf{Mobile Robot Navigation.} In the mobile robot navigation task, a ground robot moves from a start location to a target destination while avoiding static obstacles. The control problem is typically formulated in a hierarchical framework. At the high level, a learned policy $\pi_{\bm{\theta}}(\hat{p}_t\mid s_t)$ outputs intermediate waypoints, conditioned on the robot’s state, the goal position, and obstacle information. At the low level, these actions are converted into executable control inputs $(v_t, \omega_t)$ for the unicycle dynamics, ensuring feasibility under kinematic and dynamic constraints.

The detailed formulations of the lower-level optimization problems for these three environments are given as follows.

\subsection{Autonomous Vehicles Rebalancing}
The lower-level optimization policy is represented in a matrix form as follows. 
\begin{equation}
\begin{aligned}
    \min_{\bm{f}_t, \bm{g}_t} \quad & \bm{c}^\top \bm{f}_t - \bm{p}^\top \bm{g}_t + \sum_{k=1}^K \max \{ -{\bm{\theta}_f^{(k)}}^\top \bm{f}_t + {\bm{\theta}_g^{(k)}}^\top \bm{g}_t - \mu^{(k)}, 0 \} \\
    \text{s.t.} \quad & \bm{q}_t - A(\bm{f}_t + \bm{g}_t) \geq \hat{\bm{q}}_t \\
    & \bm{q}_t - G(\bm{f}_t + \bm{g}_t) \geq \bm{0} \\
    & \bm{f}_t \geq \bm{0}, \quad \bm{g}_t \geq \bm{0} \\
    & \bm{g}_t \leq \bm{\lambda}_t
\end{aligned}
\label{eq:reb}
\end{equation}
In\eqref{eq:reb}, the vectors $\bm{c}$ and $\bm{p}$ characterize the unit operational cost and unit profit, respectively. The constraints incorporate two critical matrices: the incidence matrix $A$, which governs the flow conservation and node-wise distribution updates, and the topology matrix $G$, which ensures the total outgoing flow does not exceed the available local supply. The term $\bm{\lambda}_t$ captures the real-time travel requests. Beyond the original stage cost, the parameters $\{\bm{\theta}_f^{(k)}, \bm{\theta}_g^{(k)}, \mu^{(k)}\}_{k=1}^K$ define a structured strategic cost learned via inverse optimization from expert demonstrations. 
\begin{proposition}
    The inverse feasible set for Problem \ref{eq:reb} \begin{align}
\bm{\theta}_t^{*\text{inv}}(\mathbf{f}_t, \mathbf{g}_t) := \left\{ \bm{\theta} \;\middle|\; \mathbf{f}_t, \mathbf{g}_t \in \mathcal{X}_t^{*\text{opt}}(\bm{\theta}) \right\}.
\end{align} is always non-empty, where $\mathcal{X}_t^{*\text{opt}}(\bm{\theta})$ is defined by KKT optimality conditions of Problem \ref{eq:reb}.
\end{proposition} 
Next, we will prove this proposition and also show that the solved parameters have practical significance which has not been discussed further in the general setting.
\begin{proof}
The inverse optimization solving for unknown parameters in Problem \ref{eq:reb} is formulated as follows
\begin{equation}
    \begin{gathered}
\min_{\bm{x}_t,\bm{y}_t,\bm{w}_t,\bm{v}_{1t},\bm{v}_{2t},\bm{z}_{1t},\bm{z}_{2t},\bm{\tau}_t,\bm{\theta},\mu}\sum_k\Vert(\bm{\theta}^{(k)})\Vert^2\\
         1-\bm{z}_{1t}^{(k)}-\bm{z}_{2t}^{(k)}=0,\ k = 1,...K,\ t = 1,...T\\    
        \bm{c}^T+\bm{x}_t^TA+\bm{y}_t^TG-\bm{w}_t^T-\sum_k\bm{z}_{1t}^{(k)}\bm{\theta}_f^{(k)T}+2l_1\bm{f}_t=0,\ t = 1,...T\\
         -\bm{p}^T+\bm{x}_t^TA+\bm{y}_t^TG-\bm{v}_{1t}^T+\bm{v}_{2t}^T+\sum_k\bm{z}_{1t}^{(k)}\bm{\theta}_g^{(k)^T}+2l_2\bm{g}_t=0,\ t = 1,...T\\
        \bm{x}_t^T(\bm{q}_t-A(\bm{f}_t+\bm{g}_t)-\hat{\bm{q}}_t)=0,\ t = 1,...T\\
        \bm{y}_t^T(\bm{q}_t-G(\bm{f}_t+\bm{g}_t)=0,\ t = 1,...T\\
        \bm{w}_t^T\bm{f}_t=0,\ t = 1,...T\\
        \bm{v}_{1t}^T\bm{g}_t=0,\ t = 1,...T\\
        \bm{v}_{2t}^T(\bm{g}_t-\bm{\lambda}_t)=0,\ t = 1,...T\\
        \text{(other feasiblity constraints)}
     \end{gathered}
    \label{eq:inv}
\end{equation}
Given historical data $\{\bm{f}_i, \bm{g}_i, \bm{q}_i, \bm{\lambda}_I\}_{i=1}^N$, we consider the worst-case scenario where
\[
G \bm{f}_i + G \bm{g}_i \neq \bm{q}_i, \quad \bm{f}_i \neq 0, \quad \bm{g}_i \neq 0, \quad \bm{g}_i \neq \bm{\lambda}_I \quad \forall i = 1, \dots, N.
\]
This implies that all corresponding dual variables
\[
\bm{y}_i = \bm{w}_i = \bm{v}_{1i} = \bm{v}_{2i} = 0 \quad \forall i = 1, \dots, N.
\]
We proceed with the analysis under this assumption. Moreover, We further set
\[
\bm{z}_{1i}^{(k)} = 1, \ \bm{z}_{2i}^{(k)} = 0
\]
which simplifies inverse problem corresponding to Problem \ref{eq:reb} into the following form:
\begin{equation}
    \label{eq:simplified_problem}
    \begin{aligned}
        \min_{\bm{\theta}^{(k)}, \mu^{(k)}, \bm{x}_i} \quad & \sum_k \|\bm{\theta}^{(k)}\|^2 \\
        \text{s.t.} \quad & -{\bm{\theta}_f^{(k)}}^\top \bm{f}_i + {\bm{\theta}_g^{(k)}}^\top \bm{g}_i - \mu^{(k)} = 
        \tau_i^{(k)}, \quad \forall k = 1, \dots, K, \\
        & \bm{c}^\top + \bm{x}_i^\top A - \sum_k {\bm{\theta}_f^{(k)}}^\top = 0, \\
        & -\bm{p}^\top + \bm{x}_i^\top A + \sum_k{\bm{\theta}_g^{(k)}}^\top = 0, \\
        & \text{(other feasibility constraints)}.
    \end{aligned}
\end{equation}

Now assume $\bm{x}_i \preceq \epsilon \mathbf{1}$ for some small $\epsilon > 0$. Then we have
\[
-\epsilon \mathbf{1} \preceq \bm{x}_i^\top A \preceq \epsilon \mathbf{1}.
\]
Suppose $\mathbf{c}, \mathbf{p} \in \mathbb{R}^n_{++}$. We can select $\epsilon$ such that
\[
\min_{1 \leq j \leq M} c_j \geq \epsilon, \qquad \min_{1 \leq j \leq M} p_j \geq \epsilon.
\]
Then the constraints in \ref{eq:simplified_problem} reduce to
\begin{equation}
    \label{eq:feasibility_constraints}
    \begin{aligned}
        -{\bm{\theta}_f^{(k)}}^\top \bm{f}_i + {\bm{\theta}_g^{(k)}}^\top \bm{g}_i - \mu^{(k)} &\geq 0, \quad \forall k = 1, \dots, K, \\
        \bm{c}^\top + \bm{x}_i^\top A - \sum_k{\bm{\theta}_f^{(k)}}^\top &= 0, \\
        -\bm{p}^\top + \bm{x}_i^\top A + \sum_k {\bm{\theta}_g^{(k)}}^\top &= 0.
    \end{aligned}
\end{equation}

Assume further that $K = 1$ (for simplicity). Then the last two equalities yield:
\[
\bm{\theta}_f = c + A^\top \bm{x}_i, \qquad \bm{\theta}_g = p - A^\top \bm{x}_i.
\]

In the context of autonomous vehicle rebalancing, it is reasonable to assume that the total revenue exceeds the cost, i.e., $\bm{p}^\top \bm{g}_i > \bm{c}^\top \bm{f}_i$. Then we compute:
\begin{equation}
    \label{eq:cost_revenue_diff}
    \begin{aligned}
        -\bm{\theta}_f^\top \bm{f}_i + \bm{\theta}_g^\top \bm{g}_i - \mu &= - (\bm{c}^\top + \bm{x}_i^\top A) \bm{f}_i + (\bm{p}^\top - \bm{x}_i^\top A) \bm{g}_i - \mu \\
        &= (\bm{p}^\top \bm{g}_i - \bm{c}^\top \bm{f}_i) - \bm{x}_i^\top A (\bm{f}_i + \bm{g}_i) - \mu \\
        &\geq (\bm{p}^\top \bm{g}_i - \bm{c}^\top \bm{f}_i) - \epsilon \mathbf{1}^\top (\bm{f}_i + \bm{g}_i) - \mu.
    \end{aligned}
\end{equation}

Therefore, we can choose $\epsilon$ sufficiently small and $\mu$ accordingly, such that the right-hand side of \ref{eq:cost_revenue_diff} remains non-negative. This guarantees that the constraints in \ref{eq:feasibility_constraints} are satisfied, hence the problem admits a feasible solution.
\end{proof}

Problem~\ref{eq:reb} is constructed so that the historical expert decisions $\{\bm{f}_i, \bm{g}_i\}_{i=1}^N$ are optimal solutions to the forward problem. There exist model parameters $\{\bm{\theta}_f^{(k)}, \bm{\theta}_g^{(k)}, \mu^{(k)}\}_{k=1}^K$ and desired next states $\{\hat{\bm{q}}_i\}_{i=1}^N$, which can be learned via reinforcement learning, such that the expert decisions become approximately optimal.


\subsection{Supply Chain Inventory Management}
The lower-level policy is formulated as follows
\begin{subequations}
    \begin{align}
\min_{f_{ij}^t, \bm{w}_i^t} \quad &\sum_{i \in V_W} m_i^O \cdot \bm{w}_i^t + \sum_{(i,j) \in \mathcal{E}} m_{ij}^T \cdot f_{ij}^t + \sum_{i\in V_S}\rho_i^f|\epsilon_i^f| + \sum_{i\in V_W}\rho_i^g|\epsilon_i^w|+\sum_{k=1}^K\max\{-\bm{\theta}_k^T\bm{f}^t+\mu_k,0\}+\sum_{k=1}^K\max\{-\bm{\beta}_k^T\bm{w}^t+\nu_k,0\}\\
\text{s.t.} \quad & \sum_{j \in N^-(i)} f_{ji}^t + \epsilon_i^f = \hat{q}_i^{t+1}, i \in V_S \\
& \bm{q}_i^t + \sum_{j \in N^-(i)} f_{ji}^t - d_i^t \le c_i^t, i \in V_S \\
& \sum_{j \in N^+(i)} f_{ij}^t \le \bm{q}_i^t, i \in V_W \\
& \bm{q}_i^t + w_i^t - \sum_{j \in N^+(i)} f_{ij}^t \le c_i^t, i \in V_W \\
& w_i^t + \epsilon_i^w = w_i^t, i \in V_W \\
& f_{ij}^t \ge 0, (i,j) \in \mathcal{E}, w_i^t\geq 0,i\in V_W
\end{align}
\end{subequations}

where $m^o,m^T,d,c$ are production cost, transportation cost, customer demand and capacity respectively, $w_i$ denotes order quantity in warehouse $i$ and $f_{ij}$ denotes the departing flow from warehouse $i$ to store $j$, and $q$ denotes the inventory level. Specifically, $\epsilon$ quantifies the magnitude of constraint violations, and the parameters $\beta$, $\mu$, and $\rho$ are estimated through the inverse optimization process to align the model with expert data.
\subsection{Mobile Robot}
\begin{subequations}
    \begin{align}
    \min_{\bm{u}_t=(v_t, \omega_t)} \quad & c_p \|\bm{p}_{t+1} - \hat{\bm{p}}_t\|^2 + c_h (\theta_{t+1} - \hat{\theta}_t)^2 + c_u \|\bm{u}_t\|^2 + \sum_{k=1}^K \max\{\bm{\alpha}_k^{\top} \bm{u}_t - \beta_k, 0\} \label{eq:robot_obj} \\
    \text{s.t.} \quad & x_{t+1} = x_t + T_s v_t \cos \theta_t \label{eq:robot_dyn_x} \\
    & y_{t+1} = y_t + T_s v_t \sin \theta_t \label{eq:robot_dyn_y} \\
    & \theta_{t+1} = \theta_t + T_s \omega_t \label{eq:robot_dyn_theta} \\
    & v_{\min} \le v_t \le v_{\max} \\
    & \omega_{\min} \le \omega_t \le \omega_{\max} \\
    & \|\bm{p}_{t+1} - \bm{p}_j^{\text{obs}}\| \ge r_j + d_{\text{safe}}, \quad \forall j = 1, \dots, M \label{eq:robot_safe}
\end{align}
\end{subequations}
Here, $v_t \in \mathbb{R}$ and $\omega_t \in \mathbb{R}$ denote the linear speed and angular velocity of the mobile robot, respectively. The vector $\bm{p}_t = (x_t, y_t)^\top \in \mathbb{R}^2$ represents the robot's position in the 2D Cartesian plane, while $\theta_t \in [-\pi, \pi)$ denotes its orientation. The term $\hat{\bm{p}}_t$ is the reference waypoint generated by the upper-level RL policy. In the objective function (33a), $c_p$, $c_h$, and $c_u$ are weighting coefficients for position tracking, heading alignment, and control effort. The parameters $\{\alpha_k, \beta_k\}_{k=1}^K$ characterize the structured safety cost learned via inverse optimization from expert demonstrations. In the constraints, $T_s$ is the sampling time, while $[v_{\min}, v_{\max}]$ and $[\omega_{\min}, \omega_{\max}]$ define the kinematic limits of the unicycle model. Lastly, $\bm{p}_j^{\text{obs}}$, $r_j$, and $d_{\text{safe}}$ represent the position of the $j$-th obstacle, its radius, and the required safety margin, respectively.

\section{Further Experimental Results and Discussions}
\label{app:experiment}
\subsection{RL Methodology and Experimental Setup}
\label{app:c11}
 \begin{itemize}
        \item \textbf{Autonomous Vehicle Rebalancing}: For the Autonomous Vehicle Rebalancing experiment, we train an A2C-GNN agent as in \citep{gammelli2023graph} using on-policy updates with a discount factor of $0.99$, a maximum of $50$ decision steps per episode, and $10000$ training episodes. The actor and critic share the same EdgeConv-based GNN encoder and each head is implemented as a two-layer multilayer perceptron with $256$ hidden units, and both networks are optimized with the Adam optimizer with a learning rate of $5\times 10^{-5}$. At the end of each episode we normalize the discounted returns, perform a single A2C update of actor and critic parameters with gradient-norm clipping at $5$ to improve stability, and we save the checkpoint achieving the highest cumulative training reward for later evaluation.
        \item \textbf{Inventory Management}: We adopt a graph neural network-based Advantage Actor–Critic (A2C) method as in \citep{gammelli2023graph} to learn the control policy. Training is performed with an on-policy A2C update, using a discount factor of 0.99, a maximum episode length of 30 steps, and up to 20,000 training episodes. Both actor and critic are optimized with Adam with learning rate 5e-5, and gradient norm clipping is applied to improve training stability.
        \item \textbf{Mobile Robot Navigation}: For the mobile robot, we use an off-policy Soft Actor–Critic (SAC) algorithm. SAC is trained with a replay buffer of size $10^6$, a mini-batch size of $256$, Adam optimizer with learning rate $3\times 10^{-4}$ for both actor and critic, discount factor $\gamma = 0.99$, soft target network updates with $\tau = 0.005$, and automatic entropy tuning with target entropy set to $-\text{dim}(\mathcal{A})$. Gradient updates start once the replay buffer contains more than $400$ transitions, and the agent is trained for $500$ episodes, periodically saving the actor network during training.
    \end{itemize}

\subsection{Performance under Non-Convex True Cost}
\label{sec:non_convex_cost}

To address the performance of our framework under non-convex settings, we conducted a supplementary experiment in the mobile robot navigation domain. Specifically, we introduced a non-convex obstacle avoidance penalty into the overarching true cost function, modeled as a Gaussian repulsive field---a widely adopted formulation in the robotics literature to represent smooth, non-convex spatial constraints:
\begin{equation}
    C_{obs}(p) = \sum_{i=1}^{M} A_i \exp \left( - \frac{\| p - o_i \|^2}{2\sigma_i^2} \right)
\end{equation}
where $p$ is the robot's position, $o_i$ is the obstacle's position, $A_i$ is the penalty amplitude, and $\sigma_i$ controls the field's range.

Empirical results, as summarized in Table \ref{tab:non_convex_cost}, demonstrate that the learned convex surrogate remains highly effective. Specifically, the subgoal-guided terms and the learned ReLU-based regularization terms successfully steer the robot to safely bypass obstacles and efficiently reach its destination. Notably, the baseline bi-level method lacking this ReLU-based term requires more manual effort to tune the hyperparameters within the cost function.

\begin{table}[ht]
\centering
\caption{Performance Comparison on Mobile Robot Navigation under Non-Convex True Cost}
\label{tab:non_convex_cost}
\begin{tabular}{lccc}
\toprule
\textbf{Method} & \textbf{Average Nav Time (s)} & \textbf{Average Path Length (m)} & \textbf{Energy (J)} \\
\midrule
End-to-end RL & 7.45 & 4.27 & 6.80 \\
Bi-level-original & 8.10 & 3.82 & 4.93 \\
\textbf{Bi-level-learned (Ours)} & \textbf{4.45} & \textbf{3.93} & \textbf{2.63} \\
\bottomrule
\end{tabular}
\end{table}

    \subsection{Reward-Shaping Sensitivity}
We have conducted an explicit reward-shaping sensitivity experiment on the mobile robot navigation task. By varying the weights in the reward function, we observe that our method still remains robust according to~\ref{tab:nav_cu_001},~\ref{tab:nav_cu_005} and~\ref{tab:nav_cu_01}.

\begin{table}[t]
\centering
\caption{Performance Comparison on Mobile Robot Navigation ($c_u = 0.01$).}
\label{tab:nav_cu_001}
\begin{tabular}{lccc}
\toprule
Method & Travel Time (s) & Path Length (m) & Energy (J) \\
\midrule
MPC & $3.50 \pm 0.00$ & $3.50 \pm 0.00$ & $4.81 \pm 0.00$ \\
End-to-end RL & $7.60 \pm 0.11$ & $4.29 \pm 0.02$ & $6.78 \pm 0.23$ \\
Bi-level-unchanged & $9.14 \pm 0.08$ & $4.35 \pm 0.02$ & $6.20 \pm 0.13$ \\
Bi-level-learned (Ours) & $\mathbf{4.82 \pm 0.30}$ & $\mathbf{4.16 \pm 0.19}$ & $\mathbf{2.92 \pm 1.04}$ \\
\bottomrule
\end{tabular}
\end{table}

\begin{table}[t]
\centering
\caption{Performance Comparison on Mobile Robot Navigation ($c_u = 0.05$).}
\label{tab:nav_cu_005}
\begin{tabular}{lccc}
\toprule
Method & Travel Time (s) & Path Length (m) & Energy (J) \\
\midrule
MPC & $3.50 \pm 0.00$ & $3.50 \pm 0.00$ & $4.81 \pm 0.00$ \\
End-to-end RL & $7.60 \pm 0.10$ & $4.27 \pm 0.00$ & $6.84 \pm 0.31$ \\
Bi-level-unchanged & $9.10 \pm 0.10$ & $4.31 \pm 0.02$ & $6.41 \pm 0.20$ \\
Bi-level-learned (Ours) & $\mathbf{4.40 \pm 0.10}$ & $\mathbf{4.14 \pm 0.07}$ & $\mathbf{1.52 \pm 0.06}$ \\
\bottomrule
\end{tabular}
\end{table}

\begin{table}[t]
\centering
\caption{Performance Comparison on Mobile Robot Navigation ($c_u = 0.1$).}
\label{tab:nav_cu_01}
\begin{tabular}{lccc}
\toprule
Method & Travel Time (s) & Path Length (m) & Energy (J) \\
\midrule
MPC & $3.50$ & $3.50$ & $4.81$ \\
End-to-end RL & $7.53 \pm 0.06$ & $4.28 \pm 0.01$ & $6.85 \pm 0.19$ \\
Bi-level-unchanged & $9.07 \pm 0.06$ & $4.33 \pm 0.02$ & $6.02 \pm 0.03$ \\
Bi-level-learned (Ours) & $\mathbf{4.00 \pm 0.00}$ & $\mathbf{3.93 \pm 0.01}$ & $\mathbf{2.84 \pm 0.06}$ \\
\bottomrule
\end{tabular}
\end{table}

Across all three case studies we employ standard RL methods (GNN-A2C and SAC) with conventional hyperparameters and reward designs that closely mirror the MPC stage costs. In the revised version, we will clarify that we use off-the-shelf RL implementations with fixed hyperparameters and reward shaping across all variants, so that the comparison cleanly isolates the impact of the learned lower-level FOP.

\subsection{Motivations for Inverse Optimization Approach}
We justify our choice of the proposed inverse optimization approach from three perspectives: the structural design of the cost function and the ability to extract and interpert the behavior pattern contained in expert demonstrations, and present the efficacy of the proposed inverse optimizaiton-guided RL-OC framework compared with purely imitation learning. 
\subsubsection{Structural Design Guided by Theoretical Properties}
\label{app:c31}
The primary motivation for adopting the inverse optimization framework lies in its ability to theoretically inform the structure of the lower-level cost function. Standard approaches often heuristically assume a fixed cost form (e.g., quadratic). We implement the value approximation method proposed by \citet{abdufattokhov2021learning}, where quadratic terminal cost is estimated via value functions and inverse MPC method proposed by \citep{zhang2024inverse} which assumes quadratic structure of terminal cost. To ensure a rigorous comparison, we further implement a baseline termed IMPC-bi-level, which integrates the terminal cost identified by Inverse MPC into the structured RL-OC framework. We compare the performance on autonomous vehicle rebalancing task. As table \ref{tab:av_results2} suggests, our proposed framework outperforms all three baselines. 
\begin{table}[h!]
\centering
\caption{Performance Comparison on Autonomous Vehicle (AV) Rebalancing Task.}
\label{tab:av_results2}
\begin{tabular}{@{}lcc@{}}
\toprule
\textbf{Method} & \textbf{Reward} & \textbf{Served Demand} \\
\midrule
MPC (Oracle) & 35725 ($\pm$41.6) & 3203 ($\pm$3.07) \\
\midrule
Value Approximation & 24774 ($\pm$213.5) & 2395 ($\pm$14.9) \\
IMPC & 32426 ($\pm$401.2) & 2996 ($\pm$28.6) \\
IMPC-bi-level & 33431 ($\pm$320.4) & 3044 ($\pm$25.3) \\
\textbf{Bi-level-learned (Ours)} & \textbf{35019} ($\pm$53.0) & \textbf{3141} ($\pm$6.14) \\
\bottomrule
\end{tabular}
\end{table}

\subsubsection{Interpretability}
\textbf{Interpretations of Optimal Parameters.}
The optimal parameters inferred by inverse optimization are visualized in the following figures, which show that different weights are assigned to different edges. For example, the district "0" of higher value motivates the repositioning of idle vehicles towards this area, particularly along paths with low travel costs. Thus, a significant penalty is applied to insufficient rebalancing flows on the corresponding routes. The yellow line in the figure, which denotes high inferred penalty weights, achieves this motivation. In scenario 2, the situation becomes more complex, but we can still notice some "important" routes and the "not important" ones are not penalized. 

\begin{figure*} 
    \centering 
    \begin{subfigure}[h]{0.49\linewidth} 
        \centering 
        \includegraphics[width=\linewidth]{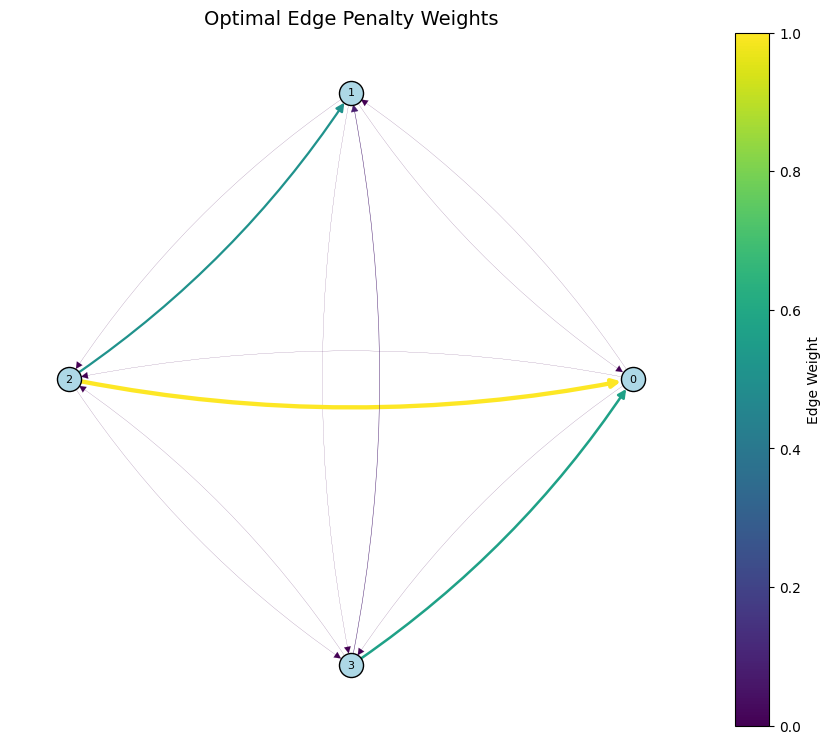} 
        \caption{Optimal Weights in Scenario 1}
        \label{fig:image-a}
    \end{subfigure}
    \hfill 
    \begin{subfigure}[h]{0.49\linewidth} 
        \centering
        \includegraphics[width=\linewidth]{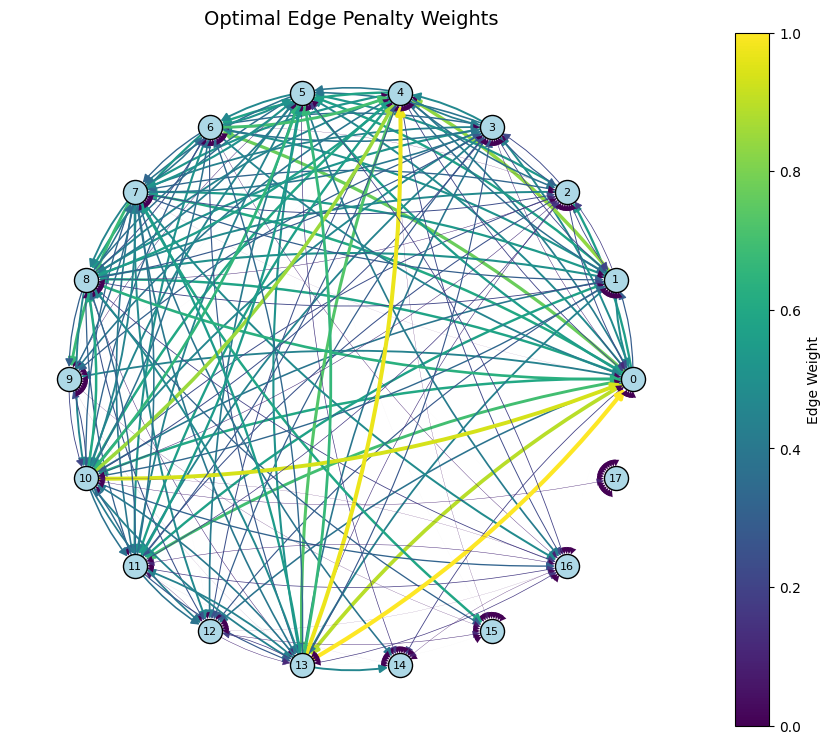} 
        \caption{Optimal Weights in Scenario 2}
        \label{fig:image-b}
    \end{subfigure}
    \caption{Visualizations of optimal parameters solved by inverse optimization}
\end{figure*}
\textbf{Interpretations of Outputs.}
To gain deeper insights into the effectiveness of the inverse method and explain model interpretability, we analyze how decisions made by two bi-level frameworks deviate from trajectories made by MPC. Table~\ref{tab:distance_metrics} shows that Bi-level-learned framework make decisions much closer to MPC trajectories, which is measured by cosine similarity and Manhattan Distance. And Figure~\ref{fig:two_tasks} shows the number of vehicles in the AV network and the inventory level over time. Similarly, our framework produces a trajectory that is closer to the MPC solution. It indicates that the learned lower-level policy based on expert demonstrations extracts information efficiently compared with the original policy with short-sighted cost functions. 

\begin{table}[h!]
    \centering
    \caption{Comparison of the cosine similarity and Manhattan distance between the rebalancing flow \( f \) obtained from different bi-level models and the flow generated by the MPC.}
    \begin{tabular}{lcc}
    \hline
    \textbf{Formulation} & \textbf{Cosine Similarity} & \textbf{Manhattan Distance} \\
    \hline
    Bi-level-original & 0.477 & 44 \\
    Bi-level-learned &\textbf{0.976} & \textbf{7} \\
    \hline
    \end{tabular}
    \label{tab:distance_metrics}
\end{table}

\begin{figure*}
        \begin{subfigure}[t]{0.49\linewidth}
        \centering
        \includegraphics[width=\linewidth, height=6cm]{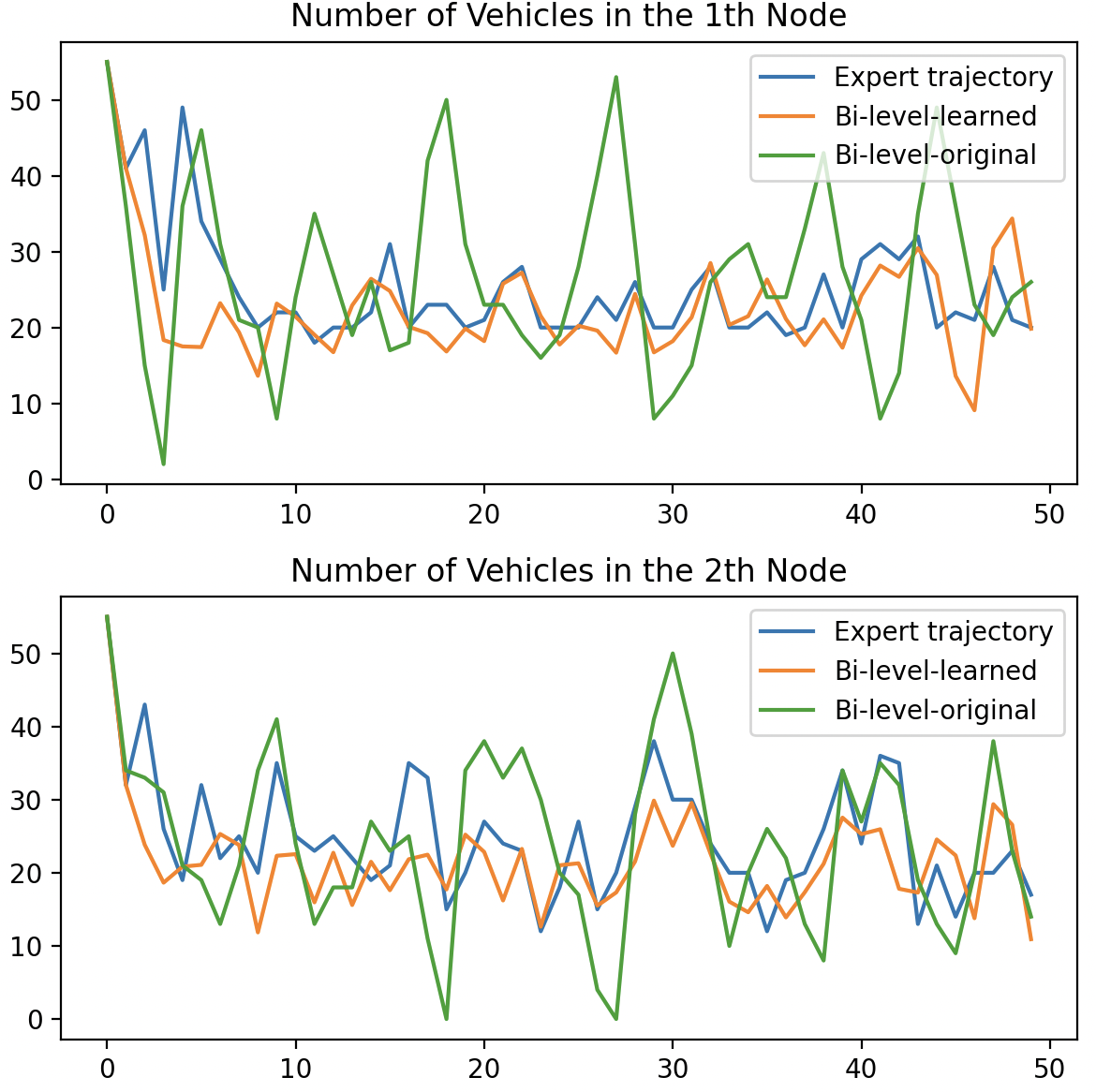}
        \caption{Number of vehicles in nodes across different methods}
        \label{fig:mobile_ori}
    \end{subfigure}\hfill
    \begin{subfigure}[t]{0.49\linewidth}
        \centering
        \includegraphics[width=\linewidth, height=6cm]{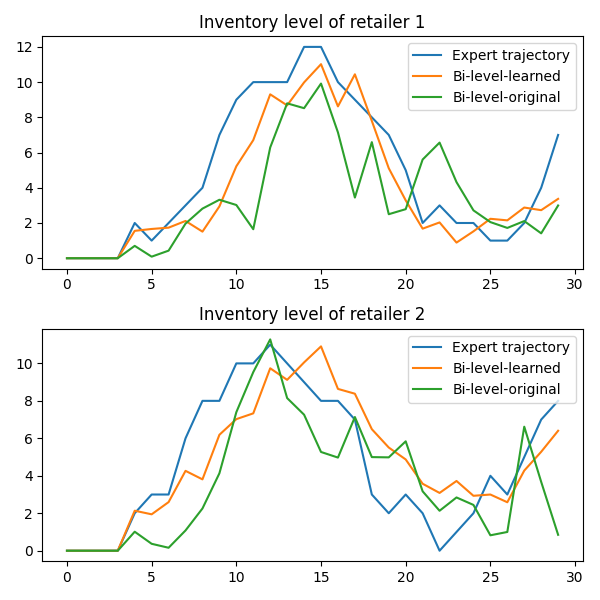}
        \caption{Inventory level of retailers in chosen nodes across different methods}
        \label{fig:mobile_learned}
    \end{subfigure}
    \caption{Comparison of system state trajectories in AV Rebalancing and SCIM across different methods}
    \label{fig:two_tasks}
\end{figure*}

\subsubsection{Data efficiency}

We further benchmark against imitation learning (IL) to demonstrate the superior sample efficiency of our approach. We implemented a set of imitation-learning baselines trained on varying numbers of expert demonstrations, and considered several supervised learning architectures as candidate policy models, including multilayer perceptrons (MLPs), Random Forest regressors, and k-Nearest Neighbors (kNN). Across the configurations we tested, these methods achieved broadly similar performance. As shown in the table~\ref{tab:imitation}, our bi-level framework consistently outperforms these imitation-learning baselines and maintains strong performance even when the number of expert demonstrations is relatively small, indicating favorable sample efficiency.

\begin{table}[H]
\centering
\caption{Performance comparison with imitation learning}
\label{tab:imitation}
\begin{tabular}{lccc}
\toprule
\textbf{Data Size} & \textbf{MPC} & \textbf{Inv-Bi-level} & \textbf{Imitation Learning} \\
\midrule
Small ($N = 100$)   & $35725(\pm 41.6)$  & $35019(\pm 53.0)$  & $33606(\pm 199.6)$ \\
Medium ($N = 500$)  & *                 & *                 & $33709(\pm 216.3)$ \\
Large ($N = 2000$)  & *                 & *                 & $33285(\pm 203.6)$ \\
\bottomrule
\end{tabular}

\vspace{0.5em}
\footnotesize
\noindent
\textbf{Note:} The asterisks (*) indicate that experiments were not required for these configurations.
The MPC baseline is independent of the training dataset size. For our Inv-Bi-level approach,
it effectively recovers the underlying cost function with minimal data, rendering larger datasets redundant.
\end{table}

\subsection{Sensitivity Analysis}
In this subsection, we analyze how the number of ReLU terms included in the cost function and the noise in expert data influence the performance of inverse optimization. We use cosine similarity and Manhattan distance between the noise-free expert decisions and decisions made by solving the lower-level optimization problem with different settings. 
\begin{itemize}
    \item Sensitivity to Number of ReLU Terms
\end{itemize}
\begin{figure}
    \centering
\includegraphics[width=0.6\linewidth]{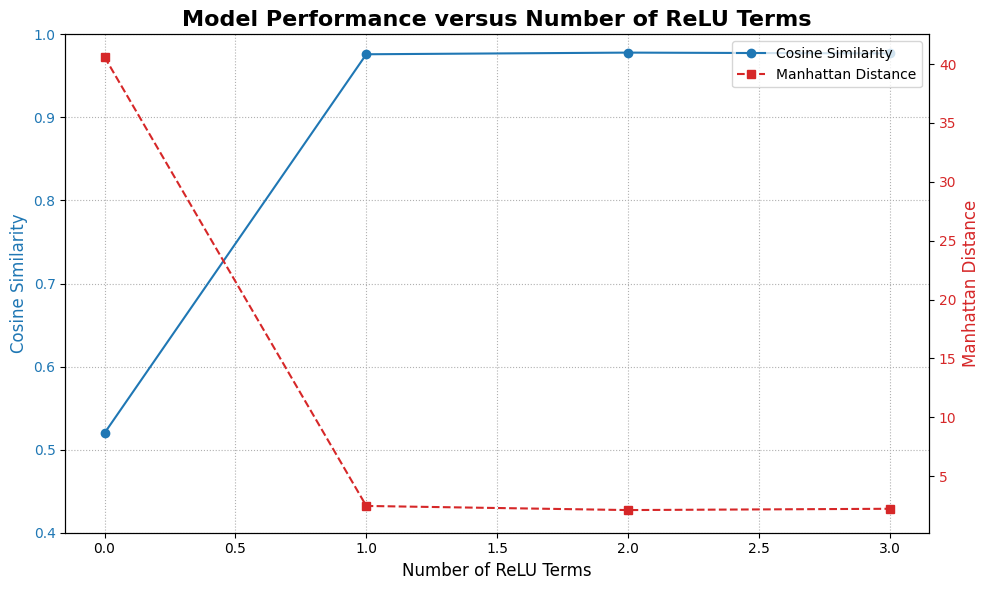}
    \caption{Sensitivity to the Number of ReLU Terms}
    \label{fig:relu}
\end{figure}
Based on Figure \ref{fig:relu}, we notice that the quality of solutions doesn't vary a lot as number of ReLU terms varies. We can infer well-performed parameters via inverse optimization utilizing only a handful of expert decisions and very few ReLU terms. 
\begin{itemize}
    \item Sensitivity to Data Noise
\end{itemize}
Sometimes given demonstrations are not optimal, thus we figure out how the noise affect inverse optimization. As figure \ref{fig:noise} shows, the solution is much closer to optimal expert demonstrations although some noise exists in offline data. Moreover, when we utilize noisy expert demonstrations to derive lower-level problem formulation and apply it in the bi-level framework, the framework improves the reward compared to the one achieved by original noisy expert demonstrations by $1.7\%$. Although the improvement is not so significant, it suggests a promising direction that we can employ our inverse optimization-guided bi-level framework to extract insights from sub-optimal offline data and make better decisions. 
\begin{figure}
    \centering
    \includegraphics[width=0.6\linewidth]{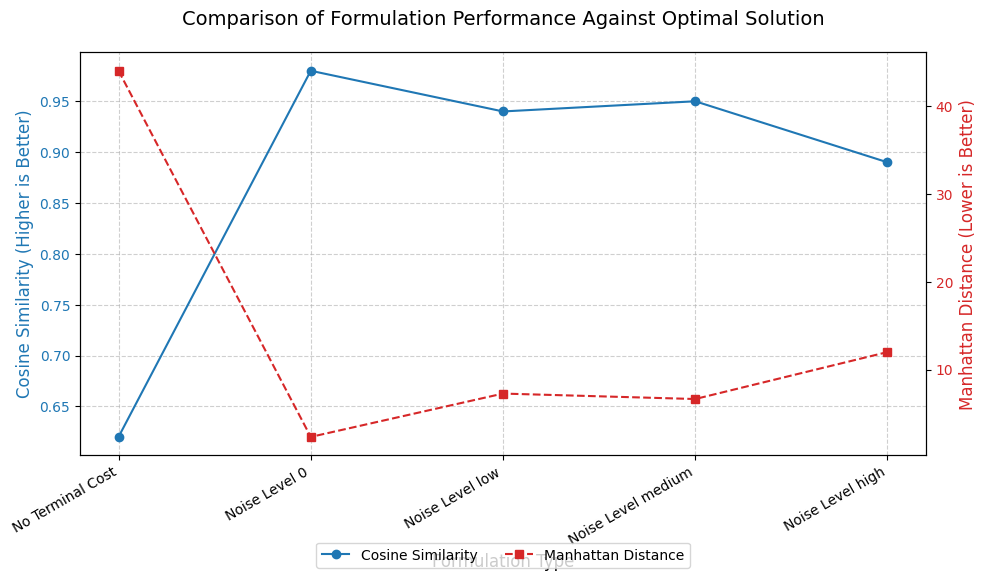}
    \caption{Sensitivity to Noise in Offline Data}
    \label{fig:noise}
\end{figure}

\subsection{Scalability Discussions}
\subsubsection{Solve-time Scaling for Inverse Optimization}
We systematically vary the problem dimension, the number of ReLU terms in the objective, and the number of expert demonstrations, and plot the resulting wall-clock solve times as shown in Figure~\ref{fig:scaling}.

\begin{figure}
    \centering
    \includegraphics[width=1.1\linewidth]{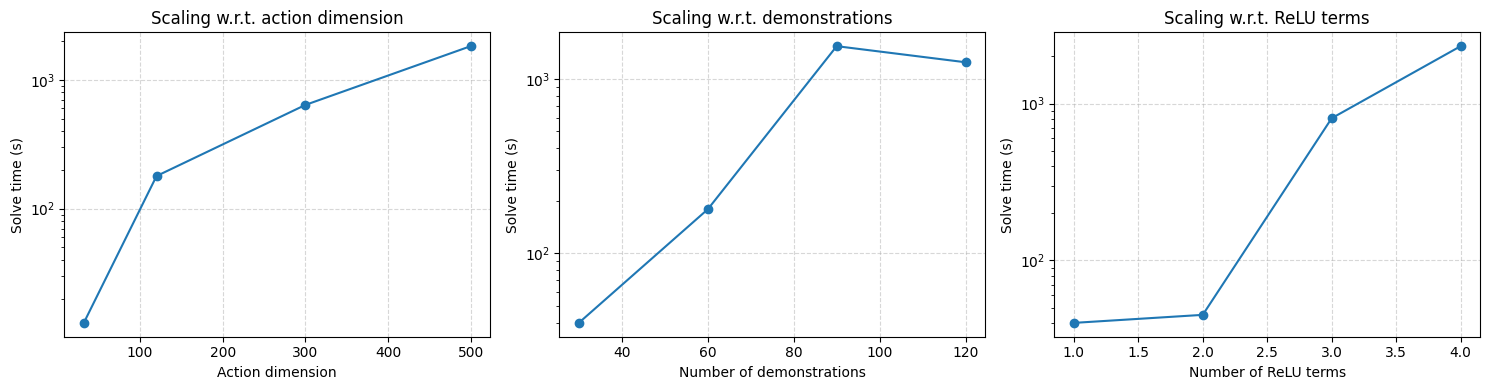}
    \caption{Solve-time Scaling Curve}
    \label{fig:scaling}
\end{figure}

We have observed that solve-time increases predictably with problem dimensions and the number of demonstrations, remaining tractable for typical system sizes without imposing prohibitive bottlenecks. The primary driver of complexity is the number of ReLU terms ($K$). While $K=1$ is computationally lightweight, increasing $K$ introduces significant combinatorial complexity. To address this, we reformulate the inverse problem as a Mixed-Integer Program (MIP) by treating the auxiliary dual variables associated with ReLU activation as binary. This reformulation is mathematically rigorous—aligning with the inherent discrete logic of ReLU functions—and leverages efficient branch-and-bound solvers to drastically accelerate convergence without compromising solution quality. We also note that solve times can vary even for instances of identical size, driven by the specific geometry of the expert demonstrations and the complexity of the active constraint set required to rationalize the data. 
\subsubsection{Learning Structured Costs at Scale}
We also evaluate a gradient-based alternative that keeps the same ReLU-based terminal-cost structure but learns its parameters by embedding the lower-level problem as a differentiable convex program via CVXPYLayers. The learned cost $d_{\theta}$ is then used in our bi-level RL-OC framework. Table~\ref{tab:cvxpy} and~\ref{tab:cvxpy_robot} reports results on autonomous vehicle rebalancing, supply chain inventory management and mobile robot navigation. Compared with solving the inverse problem to high precision—which can be challenging at larger scales—this approach enables efficient differentiating through the optimization layer. Empirically, CVXPYLayers performs slightly worse than inverse optimization but remains comparable in AV rabalancing and SCIM tasks, which highlights that our structured cost design is effective and compatible with different solution pipelines. However, performance of such gradient-based methods degrades in mobile robot navigation. This outcome reveals a distinct trade-off: while it hints at the potential of differentiable layers to solve for parameters at scale, it also implies that gradient-based methods might fail to identify underlying structures that yield generalizable insights.
\begin{table*}[t]
\centering
\caption{Performance comparison on two tasks}
\label{tab:cvxpy}
\setlength{\tabcolsep}{10pt}
\begin{tabular}{lcc@{\hspace{18pt}}lcc}
\toprule
& \multicolumn{2}{c}{\textbf{AV Rebalancing}} && \multicolumn{2}{c}{\textbf{Supply Chain Inventory Management}} \\
\cmidrule(r){2-3}\cmidrule(l){5-6}
\textbf{Method} & \textbf{Reward} & \textbf{Served Demand} &&
\textbf{Reward} & \textbf{Served Demand} \\
\midrule
MPC & 35725 ($\pm$41.6) & 3203 ($\pm$3.07) && 11335 ($\pm$34.3) & 1337 ($\pm$5.61) \\
Bi-level-cvxpy &34403 ($\pm$108.7) & 3086 ($\pm$7.06) && 9268 ($\pm$32.1) & 912 ($\pm$2.20) \\
Bi-level-learned & \textbf{35019} ($\pm$53.0) & \textbf{3141} && \textbf{9442} ($\pm$42.3) & \textbf{930} ($\pm$1.17) \\
\bottomrule
\end{tabular}
\end{table*}

\begin{table*}[h!]
\centering
\caption{Performance Comparison on Mobile Robot Navigation.}
\begin{tabular}{lccc}
\hline
\textbf{Method} & \textbf{Travel Time (s)} & \textbf{Path Length (m)} & \textbf{Energy (J)} \\
\hline
MPC                      & $3.50$ & $3.50$ & $4.81$  \\
End-to-end RL            & $7.60 \pm 0.11$ & $4.29 \pm 0.02$ & $6.78 \pm 0.23$  \\
Bi-level-cvxpy& $11.2 \pm 0.20$ & $4.57 \pm 0.03$ & $7.00 \pm 0.02$\\
Bi-level-learned (Ours)             & $\textbf{4.82} \pm 0.30$ & $\textbf{4.16} \pm 0.19$ & $\textbf{2.92} \pm 1.04$  \\
\hline
\end{tabular}
\vspace{3pt} \\
{\footnotesize An asterisk (*) indicates that no valid result was obtained for this method.}
\label{tab:cvxpy_robot}
\end{table*}

\end{document}